\documentclass{article}
\usepackage[preprint,nonatbib]{neurips_2020}
\usepackage{neurips_2020}
\usepackage{hyperref}       

\usepackage{booktabs}       
\usepackage{amsfonts}       
\usepackage{nicefrac}       
\usepackage{thm-restate}
\usepackage{microtype}      
\usepackage{xcolor}
\usepackage{graphicx}
\usepackage{enumitem}
\usepackage{amsmath,amsfonts,amsthm,bm}
\usepackage{subcaption,graphicx}
\usepackage{caption}
\usepackage{wrapfig}
\usepackage{stmaryrd}
\usepackage{amssymb,amsmath,amsthm}

\usepackage[ruled,vlined,linesnumbered]{algorithm2e}

\usepackage{etoolbox}


\SetKwInput{KwInput}{Input}                
\SetKwInput{KwOutput}{Output}
\usepackage{standalone}
\usepackage{tikz}
\usetikzlibrary{positioning}
\usetikzlibrary{calc,arrows,positioning, fit}

\usepackage{float}
\renewcommand{\[}{\begin{eqnarray}}
\renewcommand{\]}{\end{eqnarray}}

\DeclareMathOperator*{\E}{\mathbb{E}}

\usepackage{mathbbol}

\usepackage{thm-restate}

\newtheorem{proposition}{Proposition}

\DeclareMathOperator*{\argmax}{arg\,max}
\DeclareMathOperator*{\argmin}{arg\,min}
\newcommand{\indep}{\perp \!\!\! \perp}
\title{Collegial Ensembles}
\author{Authors}
\date{May 2020}
\author{
  Etai Littwin \; Ben Myara \; Sima Sabah \; Joshua Susskind \; Shuangfei Zhai\; Oren Golan\;  \\
  \texttt{Apple Inc.} \\
  \texttt{\{elittwin, bmyara, sima, jsusskind, szhai, ogolan\}@apple.com}
}
\begin{document}

\maketitle

\begin{abstract}
Modern neural network performance typically improves as model size increases. A recent line of research on the Neural Tangent Kernel (NTK) of over-parameterized networks indicates that the improvement with size increase is a product of a better conditioned loss landscape. In this work, we investigate a form of over-parameterization achieved through ensembling, where we define collegial ensembles (CE) as the aggregation of multiple independent models with identical architectures, trained as a single model. We show that the optimization dynamics of CE simplify dramatically when the number of models in the ensemble is large, resembling the dynamics of wide models, yet scale much more favorably. We use recent theoretical results on the finite width corrections of the NTK to perform efficient architecture search in a space of finite width CE that aims to either minimize capacity, or maximize trainability under a set of constraints. The resulting ensembles can be efficiently implemented in practical architectures using group convolutions and block diagonal layers. Finally, we show how our framework can be used to analytically derive optimal group convolution modules originally found using expensive grid searches, without having to train a single model.
\end{abstract}

\section{Introduction}
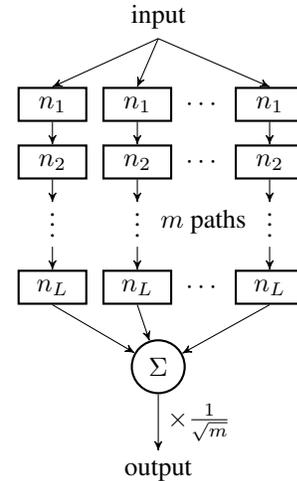
\begin{wrapfigure}[19]{r}{0.29\textwidth}
\vspace{-20pt}
\centering
\begin{tikzpicture}[->,>=stealth',auto,node distance=0.1cm,
  thick,main node/.style={draw,font=\sffamily\Large\bfseries},mylabels/.style={text=black, font=\tiny},
    every label/.append style={mylabels},
    every edge/.append style={mylabels}]
  
  \tikzstyle{cell} = [rectangle, minimum width=0.9cm, minimum height=.2cm, text centered, draw=black]
   \tikzstyle{dotscell} = [rectangle, minimum height=.1cm, text centered, draw=black]
\tikzstyle{roundcell} = [circle, text centered, draw=black, minimum size=.5cm]

\tikzstyle{blankcell} = [rectangle, minimum width=0.9cm, minimum height=0cm, text centered, draw=white]

    \node[blankcell] (N0) {};
    \node[blankcell] (x) [right=0.45cm of N0] {input};
    \node[cell] (N1) [below=.8cm of N0]{$n_1$};
    \node[cell] (N12) [right=.2cm of N1]{$n_1$};
    \node   []     (q_dots) [right=.05 of N12] {$\cdots$}; 
    \node[cell] (N13) [right=.05cm of q_dots]{$n_1$};

    \node[cell] (N2) [below=.3cm of N1]{$n_2$};
    \node[cell] (N22) [right=.2cm of N2]{$n_2$};
    \node   (q_dots2) [right=.05 of N22] {$\cdots$}; 
    
    \node[cell] (N23) [right=.05cm of q_dots2]{$n_2$};

    \node[cell] (N3) [below=1.2cm of N2]{$n_L$};
    \node[cell] (N32) [right=.2cm of N3]{$n_L$};
    \node        (q_dots3) [right=.05 of N32] {$\cdots$}; 
    \node[cell] (N33) [right=.05cm of q_dots3]{$n_L$};
    
    \node[blankcell] (t1) [below=.3cm of N2] {};
    \node[blankcell] (t2) [below=.3cm of N22] {};
    \node[blankcell] (t3) [below=.3cm of N23] {};
    \node        (q_dotsv1) [below=0.1 of N2] {$\vdots$};

    \node        (q_dotsv2) [below=0.1 of N22] {$\vdots$}; 
    \node [blankcell] (mm) [below=.3cm of q_dots2] {$m$ paths};
    \node        (q_dotsv3) [below=0.1 of N23] {$\vdots$}; 
    
    \draw [thin] (x.south) -- (N1.north); 
    \draw [thin] (x.south) -- (N12.north); 
    \draw [thin] (x.south) -- (N13.north);
    
    \draw [thin] (N1.south) -- (N2.north); 
    \draw [thin] (N12.south) -- (N22.north); 
    \draw [thin] (N13.south) -- (N23.north);

    \draw [thin] (N2.south) -- (t1.north); 
    \draw [thin] (N22.south) -- (t2.north); 
    \draw [thin] (N23.south) -- (t3.north);

    \draw [thin] (q_dotsv1.south) -- (N3.north); 
    \draw [thin] (q_dotsv2.south) -- (N32.north); 
    \draw [thin] (q_dotsv3.south) -- (N33.north); 
    \node[circle,draw=black, fill=white, inner sep=0pt, minimum size=20pt] (sum) [below=4cm of x] {$\Sigma$};
    
    \draw [thin] (N3.south) -- node [] {} (sum) ; 
    \draw [thin] (N32.south) -- node [] {}   (sum); 
    \draw [thin] (N33.south) -- node [above left] {}  (sum); 
    
    \node[blankcell] (out) [below=0.75cm of sum]{output};

     \draw [thin] (sum.south) -- node [] {$\times{1\over{\sqrt{m}}}$} (out.north); 



\end{tikzpicture}  \caption{Collegial Ensemble}
\label{fig:CE_diagram}\end{wrapfigure}
Neural networks exhibit generalization behavior in the over-parameterized regime, a phenomenon that has been well observed in practice \cite{rethinking,prove,modern,explore}. 
Recent theoretical advancements have been made to try and understand the trainability and generalization properties of over-parameterized neural networks, by observing their nearly convex behaviour at large width \cite{wide,me}. 
For a wide neural network $\mathcal{F}(x)$ with parameters $\theta$ and a convex loss $\mathcal{L}$, the parameter updates $-\mu\nabla_\theta \mathcal{L}$  can be represented in the space of functions as kernel gradient decent (GD) updates $-\mu\nabla_{\mathcal{F}}\mathcal{L}$, with the Neural Tangent Kernel (NTK) function $\mathcal{K}(x,x_j) = \nabla_\theta \mathcal{F}(x)\nabla_\theta^\top \mathcal{F}(x_j)$ operating on $x,x_j$:
\[
\Delta \theta = -\mu\nabla_\theta \mathcal{L} \longrightarrow \Delta \mathcal{F}(x) \sim -\mu\sum_j \mathcal{K}(x,x_j)\nabla_{\mathcal{F}(x_j)}\mathcal{L}
\]
In the infinite width limit, the NTK remains constant during training, and GD reduces to kernel GD, rendering the optimization a convex problem. Hence, over parameterized models in the ``large width" sense both generalize, and are simpler to optimize. In this work, we consider a different type of over-parameterization achieved through ensembling. We denote by collegial ensembles (CE) models where the output, either intermediate or final, is constructed from the aggregation of multiple identical pathways (see illustration in Fig.~\ref{fig:CE_diagram}). 
We show that the training dynamics of CE simplify when the ensemble multiplicity is large, in a similar sense as wide models, yet at a much cheaper cost in terms of parameter count. 
Our results indicate the existence of a ``sweet spot" for the correct ratio between width and multiplicity, where "close to convex" behaviour does not come at the expense of size. 
To find said ``sweet-spot", we rely on recent findings on finite corrections of gradients \cite{boris1,rtk}, though we use them in a more general manner than their original presentation. Specifically, we use the following proposition stated informally:
\begin{proposition}(Informal)\label{ass1}
Denote by $\mathcal{K}$ the NTK at initialization of a fully connected ANN with hidden layer widths $n_1,...,n_L$ and depth $L$. There exists positive constants $\alpha,C$ such that:
\[
Var(\mathcal{K}) \sim C(e^{\alpha \sum_{l=1}^Ln_l^{-1}} - 1)
\]
where the variance is measured on the individual entries of $\mathcal{K}$, with respect to the random sampling of the weights.
\end{proposition}
In \cite{boris1} and \cite{rtk}, Proposition~\ref{ass1} is proven for the on-diagonal entries of $\mathcal{K}$, in fully connected architectures. In this work, we assume and empirically verify (see Appendix Fig.~\ref{fig:alpha_fitting_resnet_block}) it holds in a more general sense for practical architectures, with different constants of $\alpha, C$. Since $Var(\mathcal{K})$ diminishes with width, we hypothesize that a small width neutral network behaves closer to its large width counterpart as $Var(\mathcal{K})$ decreases. Notably, similar observations using activations and gradient variance as predictors of successful initializations were presented in \cite{boris2,boris3}. Motivated by this hypothesis, we formulate a primal-dual constrained optimization problem that aims to find an optimal CE with respect to the following objectives:
\begin{enumerate}
    \item \textbf{Primal (optimally smooth):} minimize $Var(\mathcal{K})$ for a  fixed number of parameters.  
    \item \textbf{Dual (optimally compact):}~~minimize number of parameters for a fixed $Var(\mathcal{K})$.
\end{enumerate}
The primal formulation seeks to find a CE which mimics the simplified dynamics of a wide model using a fixed budget of parameters, while
the dual formulation seeks to minimize the number of parameters without suffering the optimization and performance consequences typically found in the "narrow regime". Using both formulations, we find excellent agreement between our theoretical predictions and empirical results, on both small and large scale models. 

Our main contributions are the following:
\begin{enumerate}
    \item We adapt the popular over-parameterization analysis to collegial ensembles (CE), in which the output units of multiple architecturally identical models are aggregated, scaled, and trained as a single model. 
    For ensembles with $m$ models each of width $n$, we show that under gradient flow and the L2 loss, the NTK remains close to its initial value up to a $\mathcal{O}\big((mn)^{-1}\big)$ correction. 
    \item We formulate two optimization problems that seek to find optimal ensembles given a baseline architecture, in the primal and dual form, respectively. The optimally smooth ensemble achieves higher accuracy than the baseline, using the same budget of parameters. The optimally compact ensemble achieves a similar performance as the baseline, with significantly fewer trainable parameters.
    \item We show how optimal grouping in ResNeXt \cite{resnext} architectures can be derived and improved upon using our framework, without the need for an expensive search over hyper-parameters.
\end{enumerate}
The remaining paper is structured as follows. In Sec.~\ref{sec:3} we formally define collegial ensembles, and present our results for their training dynamics in the large $m,n$ regime. In Sec.~\ref{sec:efficient_ensembles} we present our framework for architecture search in the space of collegial ensembles, and in Sec.~\ref{sec:5} we present further experimental results on the CIFAR-10/CIFAR-100 \cite{cifar} and ImageNet \cite{imagenet} datasets using large scale models. 



\section{Collegial Ensembles}\label{sec:3}
We dedicate this section to formally define collegial ensembles, and analyze their properties from the NTK perspective. Specifically, we would like to formalize the notion of the "large ensemble" regime, where its dynamic behaviour is reminiscent of wide single models. 
In the following analysis we consider simple feed forward fully connected neural network  $\mathcal{F}_\bold{n}(x,\theta): \mathbb{R}^{n_0} \to \mathbb{R}$, where the width of the hidden layers are given by $\bold{n} = \{n_l\}_{l=1}^{L} \in \mathbb{Z}_+^L$, adopting the common NTK parameterization:
\[\label{eq:vanilla}
\mathcal{F}_\bold{n}(x,\theta) = \sqrt{\frac{1}{n_L}}\theta^L(...\phi(\sqrt{\frac{2}{n_1}}\theta^2\phi(\sqrt{\frac{2}{n_0}}\theta^1 x))).
\]
where $x \in \mathbb{R}^{n_0}$ is the input, $\phi(\cdot)$ is the ReLU activation function, $\theta^l$ is the weight matrix associated with layer $l$, and $\theta$ denotes the concatenation of all the weights of all layers, which are initialized iid from a standard normal distribution. Given a dataset $X = \{x_i\}_{i=1}^N$, the empirical NTK denoted by $\mathcal{K}_\bold{n}(\theta) \in \mathbb{R}^{N\times N}$ is given by
$\mathcal{K}_\bold{n}(\theta) = \nabla_\theta \mathcal{F}_\bold{n}(\theta)\nabla_\theta^\top \mathcal{F}_\bold{n}(\theta),$
where $\mathcal{F}_\bold{n}(\theta) = [\mathcal{F}_\bold{n}(x_1,\theta),...,\mathcal{F}_\bold{n}(x_N,\theta)]^\top$.
Given the network $\mathcal{F}_\bold{n}$, we parameterize a space of ensemble members $\mathcal{F}^e(\Theta)$ by a multiplicity parameter $1\leq m$ and a width vector $\bold{n}$,  such that:
\[\label{eq:ensemble_ntk}
\mathcal{F}^e(\Theta) = \frac{1}{\sqrt{m}}\sum_{j=1}^m \mathcal{F}_{\bold{n}}(\theta_j),~~~~\mathcal{K}^e(\Theta) = \frac{1}{m}\sum_{j=1}^m \mathcal{K}_{\bold{n}}(\theta_j)
\]
where $\Theta = [\theta_1^\top...\theta_m^\top]^\top$ is the concatenation of the weights of all the ensembles, and $\mathcal{K}^e(\Theta) = \nabla_{\Theta}\mathcal{F}^e\nabla^\top_{\Theta}\mathcal{F}^e$ is the NTK of the ensemble. Plainly speaking, the network $\mathcal{F}_\bold{n}$ defines a space of ensembles given by the scaled sum of $m$ neural networks of the same $\mathcal{F}_\bold{n}$ architecture, with weights initialized independently from a normal distribution.
Since each model $\mathcal{F}_{\bold{n}}(\theta_j)$ in the ensemble is architecturally equivalent to $\mathcal{K}_\bold{n}(\theta)$, it is easy to show that the infinite width kernel is equal for both models:
$
\mathcal{K}_\infty = \lim_{min(\bold{n}) \to  \infty}\mathcal{K}_\bold{n}(\theta) = \lim_{min(\bold{n}) \to \infty}\mathcal{K}^e(\Theta).
$
We define the Neural Mean Kernel (NMK) $\mathcal{K}_\bold{n}^\infty$ as the mean of the empirical NTK: 
\[
\mathcal{K}_\bold{n}^\infty = \E_\theta[\mathcal{K}_\bold{n}(\theta)]].
\]


The NMK is defined by an expectation over the normally distributed weights, and does not immediately equal the infinite width limit of the NTK given by $\mathcal{K}_\infty$. 
The following Lemma stems from the application of the strong law of large numbers (LLN):
\begin{restatable}[Infinite ensemble]{lemma}{ensemblea}\label{lem:inf_ensemble}
The following holds:
\[
\mathcal{K}^e(\Theta) \stackrel{a.s}{\longrightarrow} \mathcal{K}_\bold{n}^\infty ~~~\textnormal{ as $m \to \infty$}. 
\]
\end{restatable}
We deffer the reader to Sec.~\ref{proofs} in the appendix for the full proof.
While both $\mathcal{K}_\bold{n}^\infty$ and $\mathcal{K}_\infty$ do not depend on the weights, they are defined differently. $\mathcal{K}_\bold{n}^\infty$ is defined by an expectation over the weights, and depends on the width of the architecture, whereas  $\mathcal{K}_\infty$ is defined by an infinite width limit. However, empirical observation using Monte Carlo sampling, as presented in Fig.~\ref{fig:NMK2NTK}, show little to no dependence of the NMK on the widths $\bold{n}$. Moreover, we empirically observe that $\mathcal{K}_\bold{n}^\infty \sim \mathcal{K}_\infty$ (Note that similar observations have been reported in \cite{NTK}).  
\begin{figure*}[t]
\centering
  \subcaptionbox{\label{fig:NMK2NTKa}}{\includegraphics[height=1.7in]{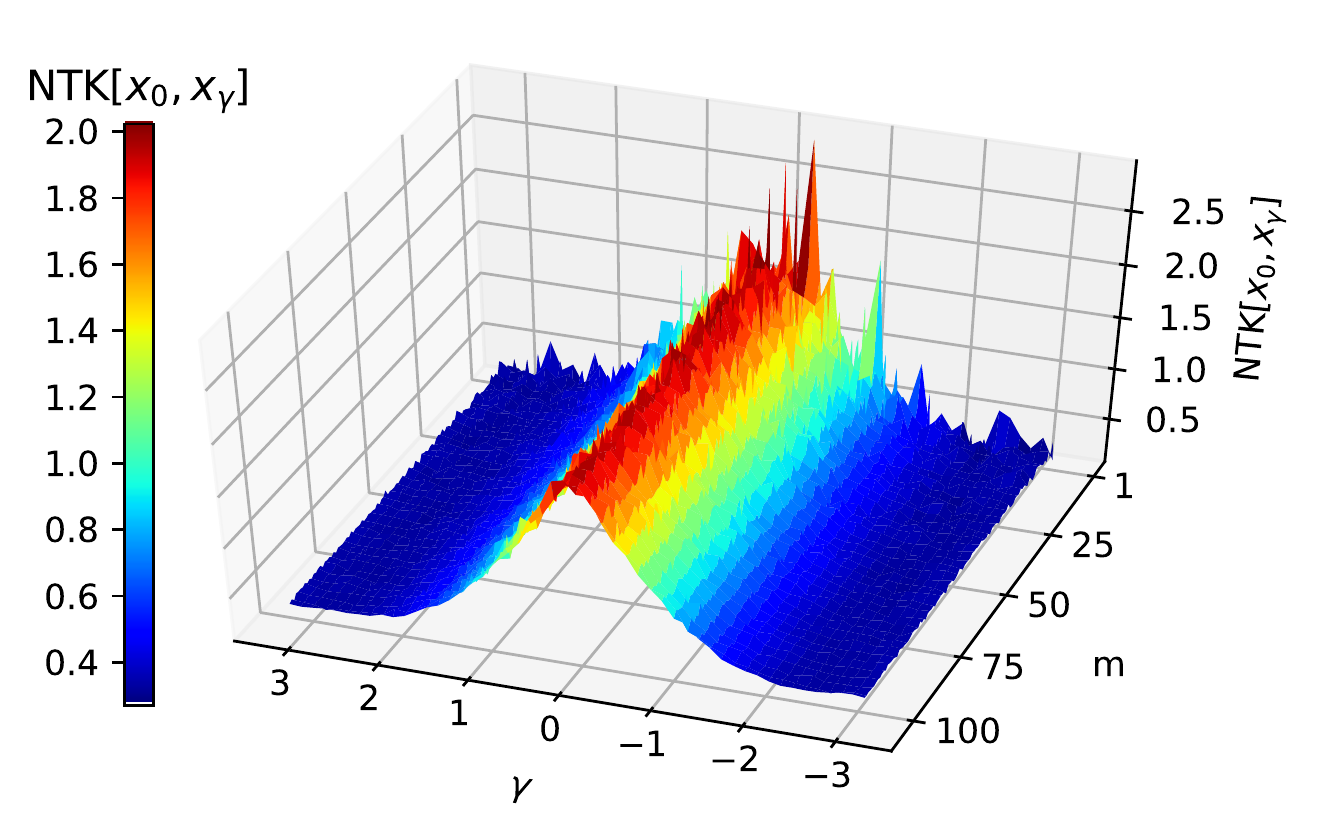}}\hspace{1em}%
  \subcaptionbox{\label{fig:NMK2NTKb}}{\includegraphics[height=1.7in]{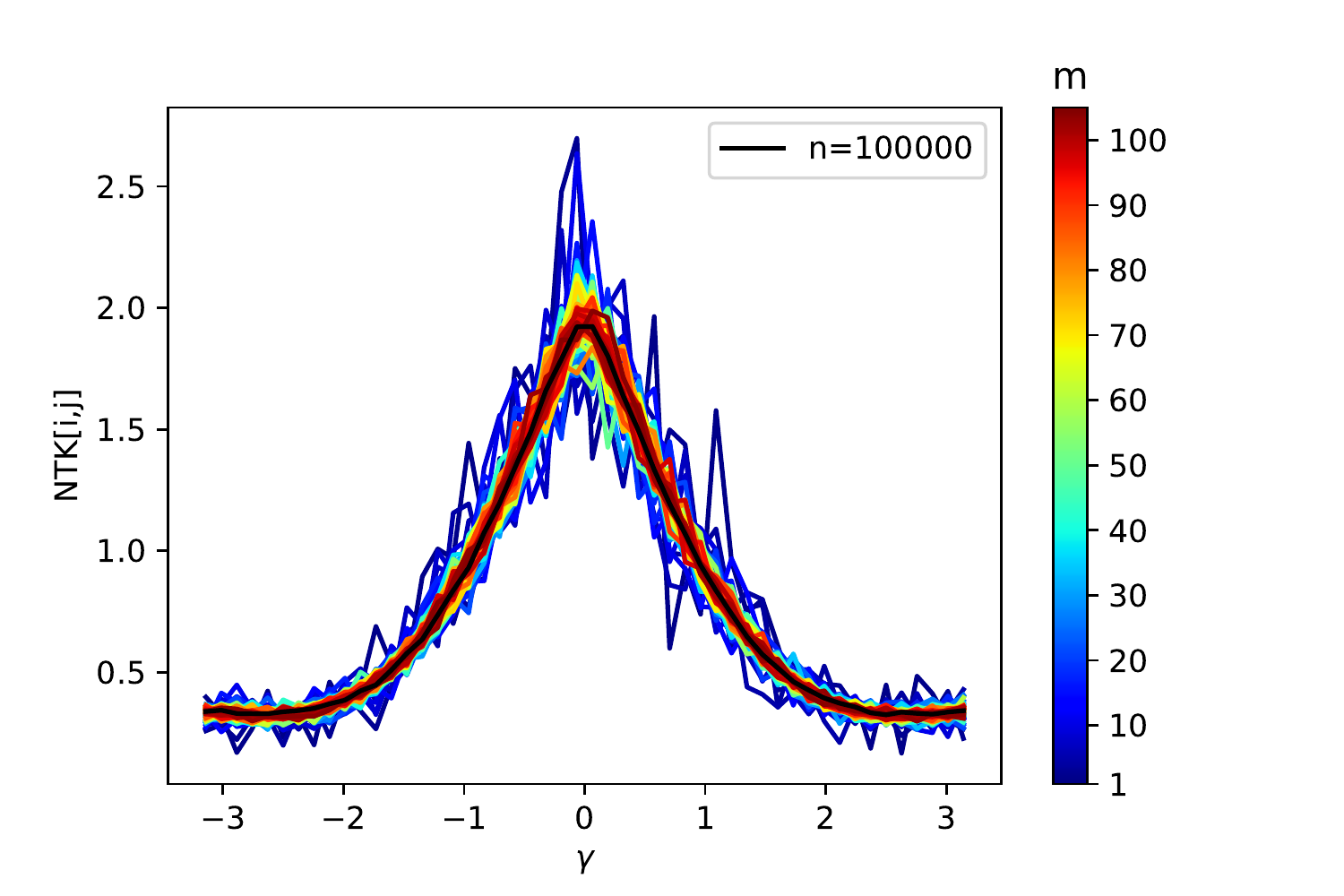}}\\
  \caption{\textbf{Convergence of the ensemble NTK to the NMK} when increasing the number of models in the ensemble. $NTK[x_0,x_{\gamma}]$ is computed for both the diagonal and
  off-diagonal elements $x_{\gamma}=[\cos(\gamma), \sin(\gamma)]$ for $\gamma\in[-\pi,\pi]$. The smoother surface as $m$ increases in (a) demonstrates the convergence of the NTK. The black line in (b), computed with $\times100$ wider model, shows that the convergence is indeed to the NMK, and the NTK mean does not depend on the width of the model. Each model in the ensemble is fully connected with $L=4$ layers and $n=1000$. }
  \label{fig:NMK2NTK}
\end{figure*}
We next show that under gradient flow, $\mathcal{K}^e(\Theta)$ remains close to its initial value for the duration of the optimization process when $mn$ is large.
Given the labels vector $\bold{y}\in\mathbb{R}^N$ and the $L_2$ cost function at time $t$, $\mathcal{L}_t = \frac{1}{2}\|\mathcal{F}^e(\Theta_t) - \bold{y}\|_2^2$, under gradient flow with learning rate $\mu$, the weights evolve over continuous time according to $\dot{\Theta_t} = -\mu\nabla^\top_{\Theta_t}\mathcal{L}_t$. 
The following theorem gives an asymptotic bound on the leading correction of the ensemble NTK over time. For simplicity, we state our result for constant width networks.

\begin{restatable}[NTK evolution over time]{theorem}{ensembleb}\label{thm:evolution}
For analytic activation functions $\phi(\cdot)$ and the $L_2$ cost function , it holds for any finite $t$:
\[
\mathcal{K}^e(\Theta_t) - \mathcal{K}^e(\Theta_0)  \sim  \mathcal{O}_p(\frac{1}{mn}) 
\]
where the notation $x_n = \mathcal{O}_p(y_n)$ states that $x_n/y_n$ is stochastically bounded.
\end{restatable}
We verified the result of Theorem~\ref{thm:evolution} for fully connected networks trained on MNIST, results are summarized in Fig.~\ref{fig:dynamics} in the appendix, and the full proof is in Sec.~\ref{proofs} in the appendix.
Large collegial ensembles therefore refer to a regime where $mn$ is large. In the case of infinite multiplicity, optimization dynamics reduces to kernel gradient descent with $\mathcal{K}_\bold{n}^\infty$, rather than $\mathcal{K}_\infty$ as the relevant kernel function. A striking implication arises from Theorem~\ref{thm:evolution}.
The total parameter count in the ensemble is linear in $m$, and quadratic in $n$, hence it is much more parameter efficient to increase $m$ rather than $n$. 
Since the "large" regime depends on both $n$ and $m$, 
CE possess an inherent degree of flexibility in their practical design.
As we show in the following section, this increased flexibility allows the design of both parameter efficient ensembles, and better performing ensembles, when compared with the baseline model.



\section{Efficient Ensembles}\label{sec:efficient_ensembles}
In this section, we use Proposition~\ref{ass1} to derive optimally smooth and compact ensembles. We parameterize the space of ensembles using $m,\bold{n}$, and a baseline architecture $\mathcal{F}_{\bold{\tilde{n}}}$, where $\bold{\tilde{n}}$ is the width vector of the baseline model. Denote by $\beta(\bold{n})$ the total parameter count in  $\mathcal{F}_{\bold{n}}$, we define parameter efficiency $\rho(m,\bold{n})$ by the ratio between the parameter count in the baseline model $\beta_s \triangleq \beta(\bold{\tilde{n}})$,  and the parameter count in the ensemble given by $\beta_e \triangleq m\beta(\bold{n})$:
\[
\rho(m,\bold{n}) \triangleq \frac{\beta_s}{\beta_e} = \frac{\beta_s}{m\beta(\bold{n})}.
\label{eq:rho}
\]

Using Proposition.~\ref{ass1}, the variance of $\mathcal{K}_{\bold{n}}$ as a function of widths $\bold{n}$ and depth $L$, is given by:
\[\label{eq:var0}
Var\big(\mathcal{K}_{\bold{n}}(\theta)\big)  \sim (e^{\alpha \sum_{l=1}^L n_l^{-1}} - 1).
\]
for some value of $\alpha$.

\paragraph{Primal formulation:}
We cast the primal objective as an optimization problem, where we would like to find parameters $m^p,\bold{n}^p$ that correspond to the smoothest ensemble:
\[\label{eq:obj}
m^p,\bold{n}^p = \argmin_{m,\bold{n}}Var\Big(\mathcal{K}^e(\Theta)\Big)~~~ s.t ~~~\rho(m,\bold{n}) = 1.
\]

Since the weights for each model are sampled independently, it holds that:
\[\label{eq:var}
Var\big(\mathcal{K}^e(\Theta)\big)  = \sum_{j=1}^m \frac{1}{m^2}Var\big(\mathcal{K}_{\bold{n}}(\theta_j)\big) = \frac{(e^{\alpha \sum_{l=1}^Ln_l^{-1}} - 1)}{m}.
\]

Equating the parameter count in both models to maintain a fixed efficiency, we can derive for each $\bold{n}$ the number of the models $m^p(\bold{n})$ in the primal formulation:
\[\label{primal}
m^p(\bold{n}) = \frac{\beta_s}{\beta(\bold{n})}~~~\longrightarrow \bold{n}^p = \argmin_{\bold{n}}  \Big[\frac{(e^{\alpha \sum_{l=1}^L n_l^{-1}} - 1)}{m^p(\bold{n})}\Big].
\]
The optimal parameters $\bold{n}^p$ can be found using a grid search.
\paragraph{Dual formulation:}The dual formulation can be cast as an optimization problem, with the following objective:
\[\label{eq:dual}
m^d,\bold{n}^d = \argmax_{m,\bold{n}}\rho(m,\bold{n}) ~~~ s.t ~~~Var\big(\mathcal{K}^e(\Theta)\big) = Var\big(\mathcal{K}_\bold{\tilde{n}}(\theta)\big).
\]
Matching the smoothness criterion using Eq.~\ref{eq:var}, we can derive for each $\bold{n}$ the number of models $m^d(\bold{n})$ in the dual formulation:
\[\label{dual}
 m^d(\bold{n}) = \frac{(e^{\alpha \sum_{l=1}^L n_l^{-1}} - 1)}{(e^{\alpha \sum_{l=1}^L \tilde{n}_l^{-1}} - 1)}~~~\longrightarrow \bold{n}^d = \argmax_{\bold{n}} \Big[\frac{1}{m^d(\bold{n})}\frac{\beta_s}{\beta(\bold{n})}\Big].
\]
Ideally, we can find $m^d,\bold{n}^d$ such that the total parameter count in the ensemble is considerably reduced. 
Equating the solutions for both the primal and dual problems in Eq.~\ref{primal} and Eq.~\ref{dual}, it is straightforward to verify that $\bold{n}^d = \bold{n}^p$, implying strong duality. Therefore, the primal and dual solutions differ only in the optimal multiplicity $m(\bold{n})$ of the ensemble. 
Both objectives are plotted in Fig.~\ref{fig:dualthm} using a feedforward fully connected network baseline with $L=6$ and constant width $\tilde{n}=500$.

Note that the efficient ensembles framework outlined in this section can readily be applied with different efficiency metrics. For instance, instead of using the parameter efficiency, one could consider the FLOPs efficiency (see Appendix~\ref{appendix:flops_efficiency}).

\begin{figure*}[t]
\centering
\subcaptionbox{Primal formulation\label{fig3:a}}{\includegraphics[width=.45\textwidth]{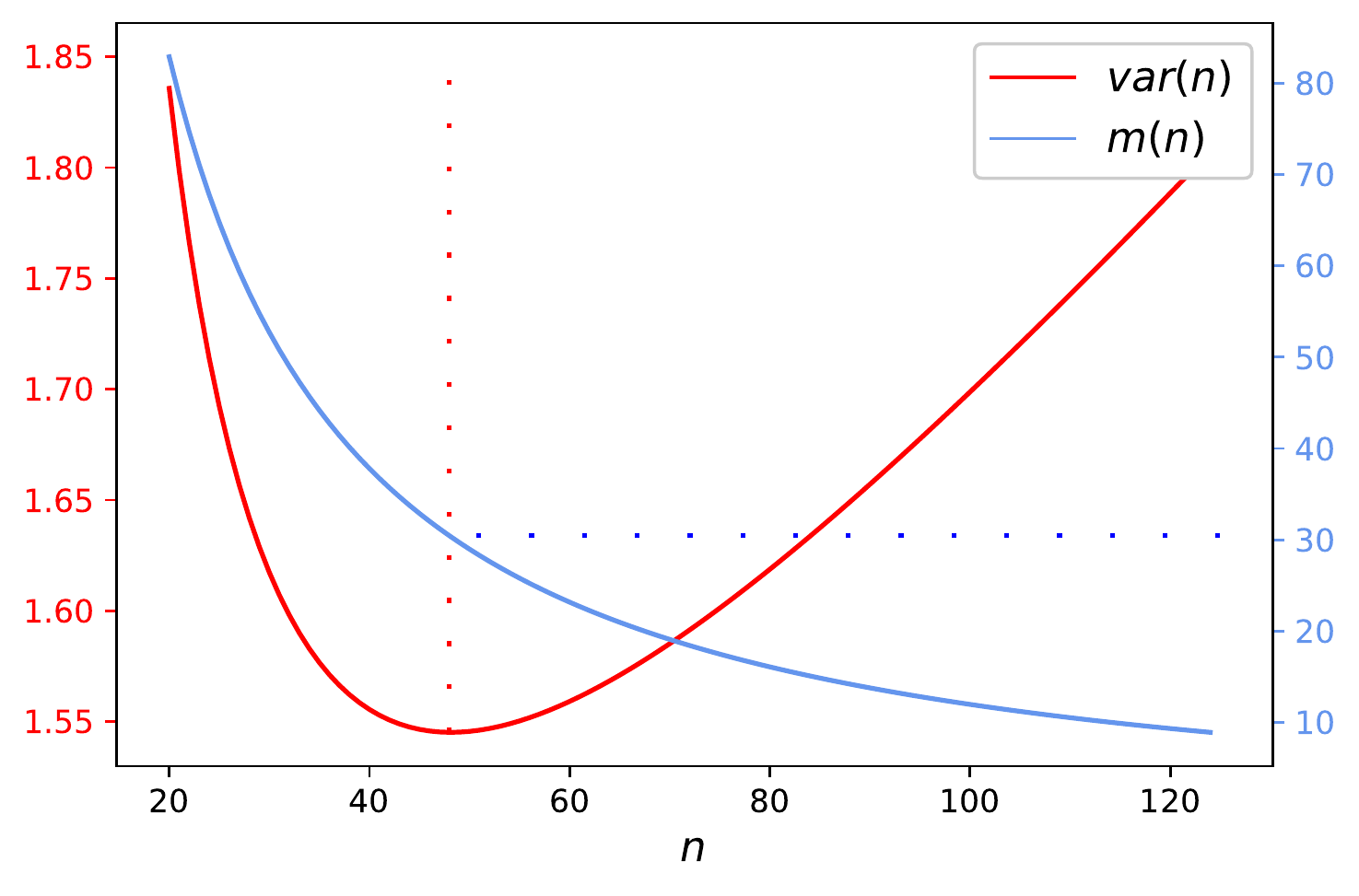}}\hfill
  \subcaptionbox{Dual formulation\label{fig3:b}}{\includegraphics[width=.45\textwidth]{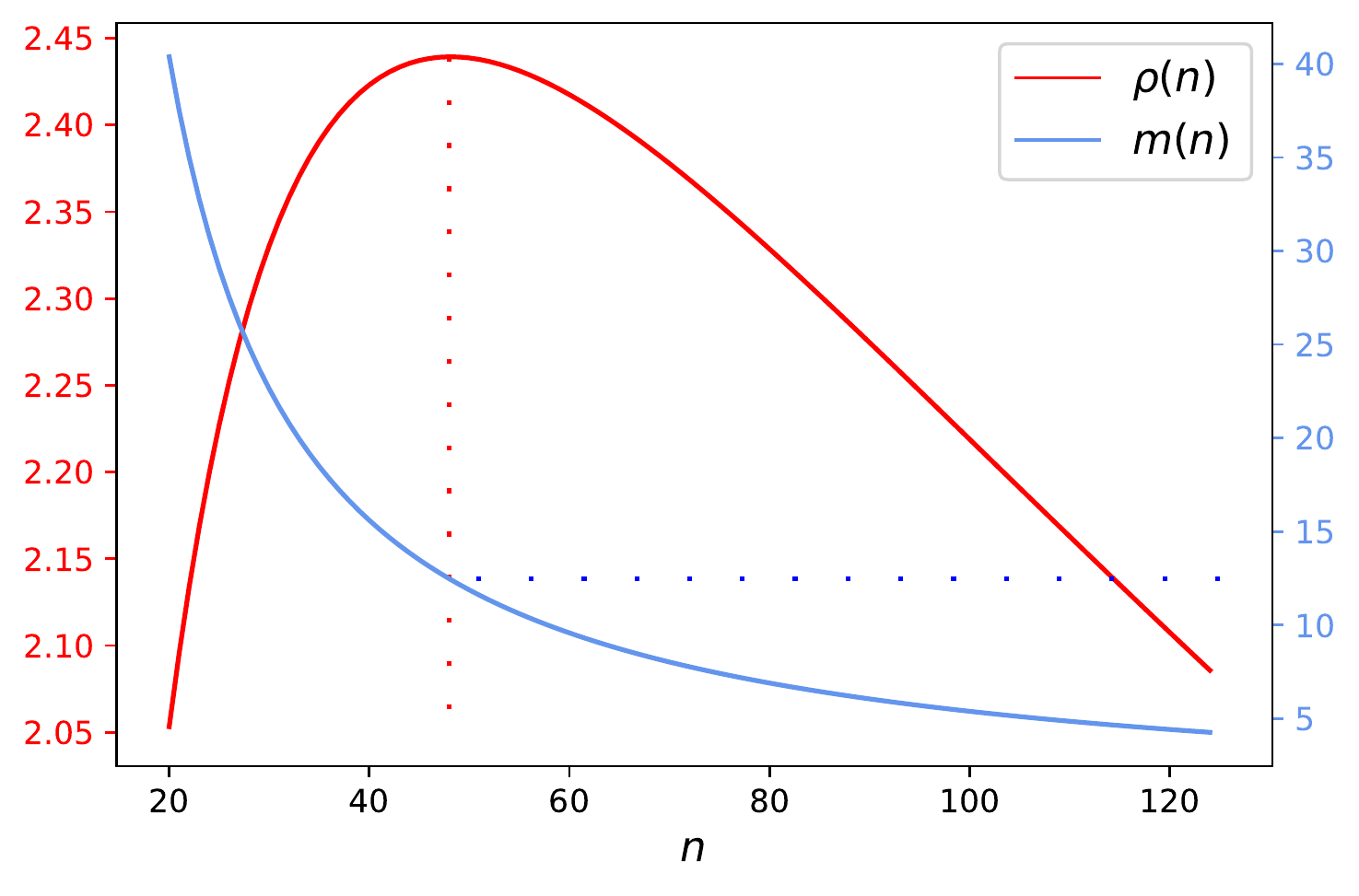}}
  \caption{\textbf{Primal and dual objective curves} for a baseline feedforward fully connected network with $L=6$ layers, $\tilde{n}=500$, and $n_0=748$. (a) The minimizer of the primal objective (red) is achieved for $ n=48$ and $m(48)\approx30$. (b) The maximizer of the dual objective (red) is achieved for $n=48$ and $m(48)\approx12$, achieving an efficiency value $\rho(48)\approx2.45$.}
  \label{fig:dualthm}
\end{figure*}

\section{Experiments}\label{sec:5}
In the following section we conduct experiments to both validate our assumptions, and evaluate our framework for efficient ensemble search. Starting with a toy example, we evaluate the effect of $Var(\mathcal{K})$ and $\beta_e$ on test performance using fully connected models trained on the MNIST dataset. For the latter experiments, we move to larger scale models trained on CIFAR-10/100 and the ImageNet~\cite{imagenet} datasets.
\subsection{Ablation Study – MNIST}

An effective way to improve the performance of a neural network  is to increase its size. Recent slim architectures, such as ResNeXt, demonstrate it possible to maintain accuracy while significantly reducing parameter count.
In Fig.~\ref{fig:primaldualthm} we provide further empirical evidence that capacity of a network by itself is not a good predictor of performance, when decoupled from other factors.

Specifically, we demonstrate strong correlation between the empirical test error and the variance $Var(\mathcal{K)}$, while $\beta_e$ is kept constant (primal). On the other hand, increasing $\beta_e$ while keeping $Var(\mathcal{K})$ constant  (dual) does not improve the performance. 
For both experiments we use as a baseline a fully connected model with $L=6$ layers and width $\tilde{n}=200$ for each layer. The width of a layer for each of the $m$ models in the ensemble is $n$. Each ensemble was trained on MNIST for 70 epochs with the Adam optimizer, and the accuracy was averaged over 100 repetitions.

\begin{figure*}[!h]
\centering
  \subcaptionbox{Primal\label{fig:mnist_primal_var}}{\includegraphics[width=.24\textwidth]{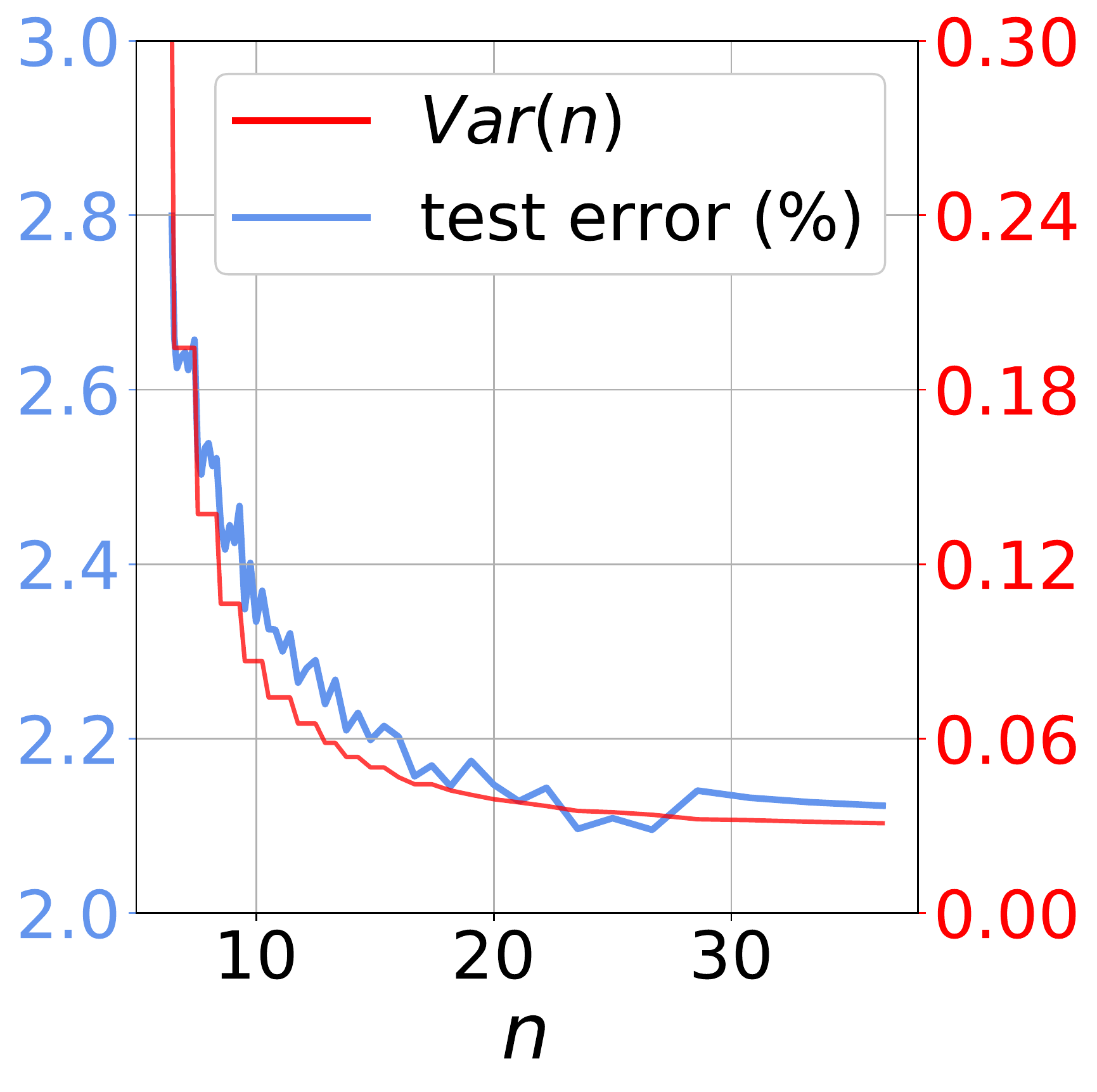}}\hfill 
  \subcaptionbox{Dual\label{fig:mnist_dual_var}}{\includegraphics[width=.24\textwidth]{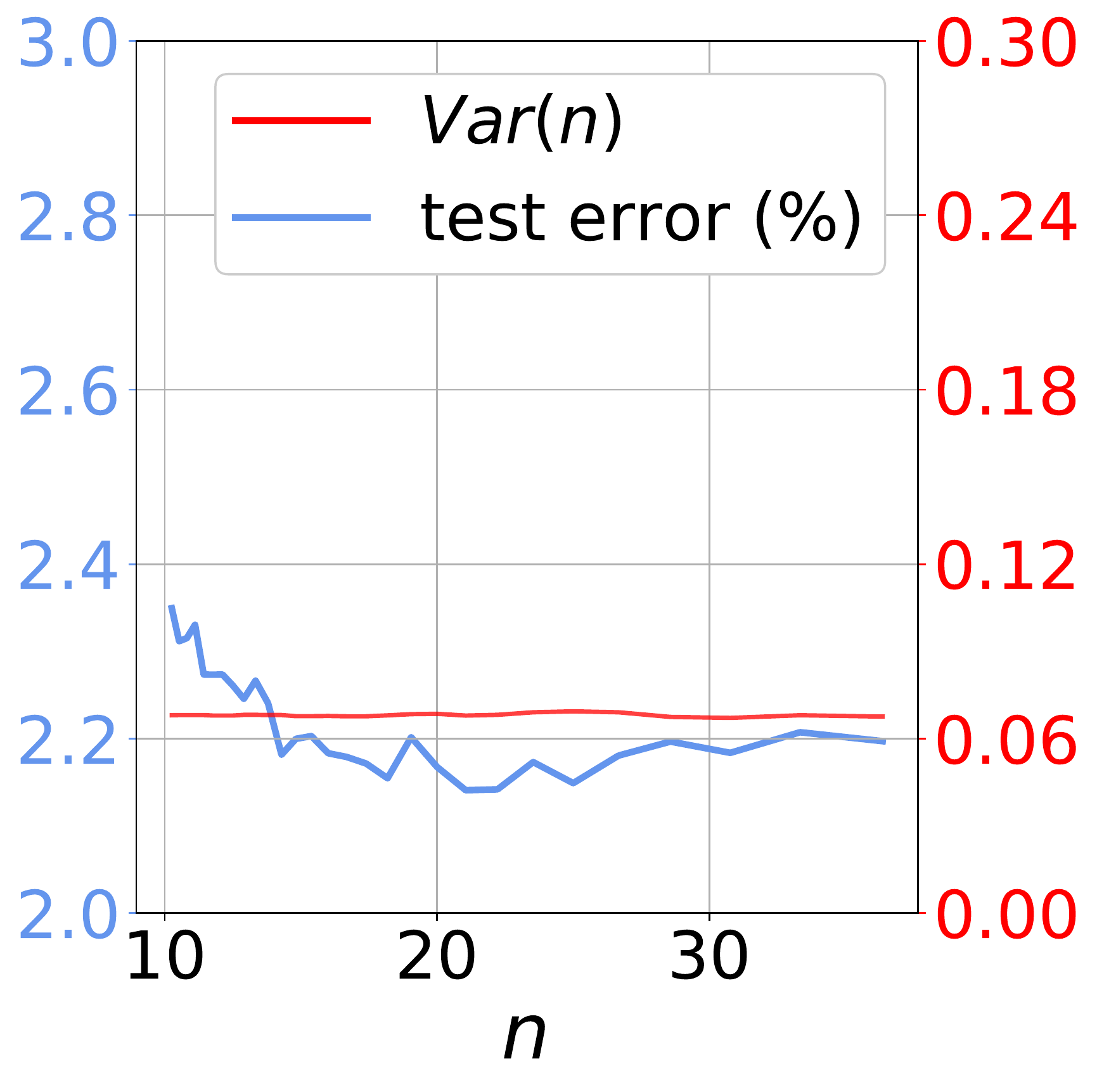}}
  \subcaptionbox{Primal\label{fig:mnist_primal_beta}}{\includegraphics[width=.24\textwidth]{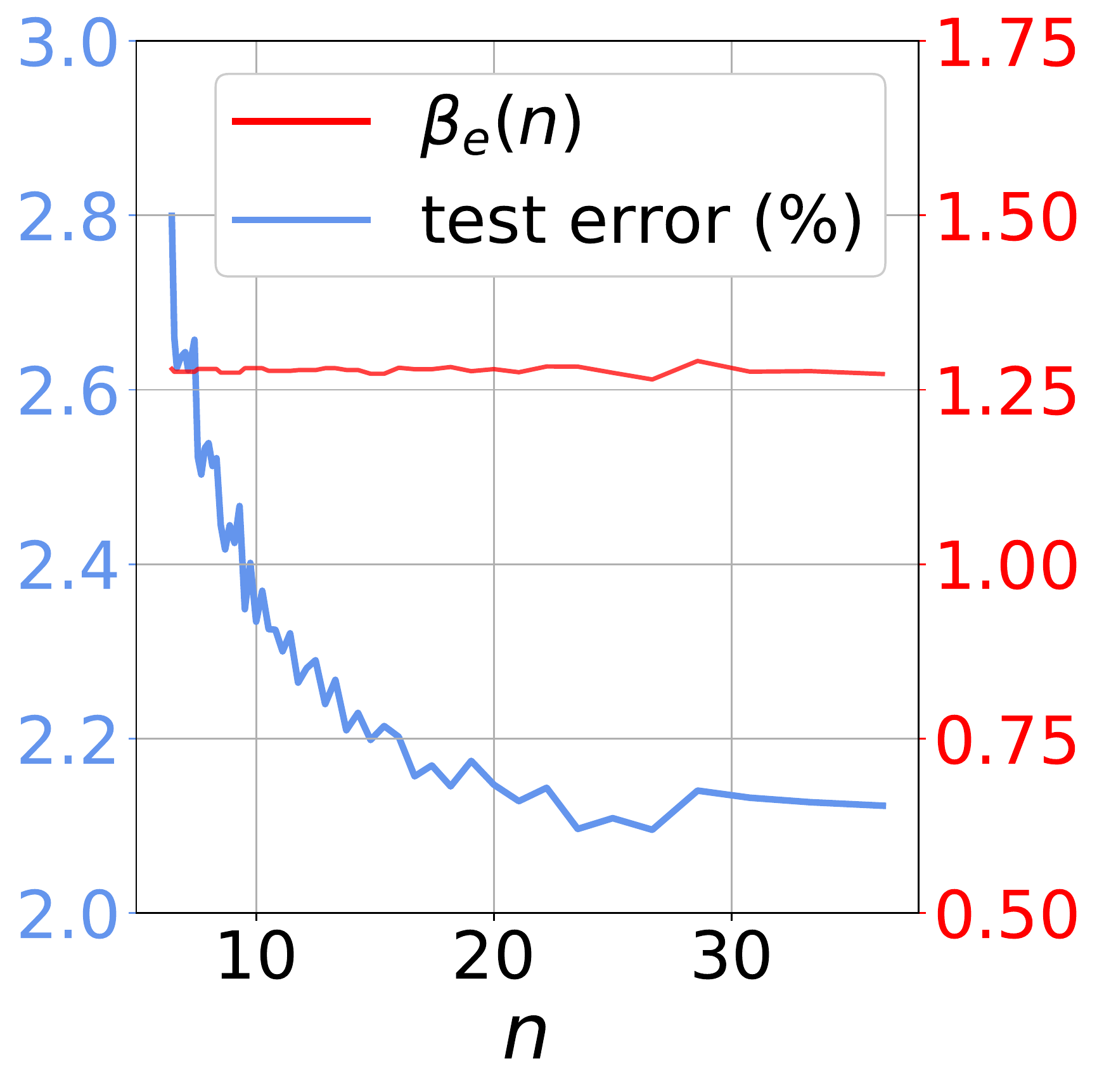}}\hfill 
  \subcaptionbox{Dual\label{fig:mnist_dual_beta}}{\includegraphics[width=.24\textwidth]{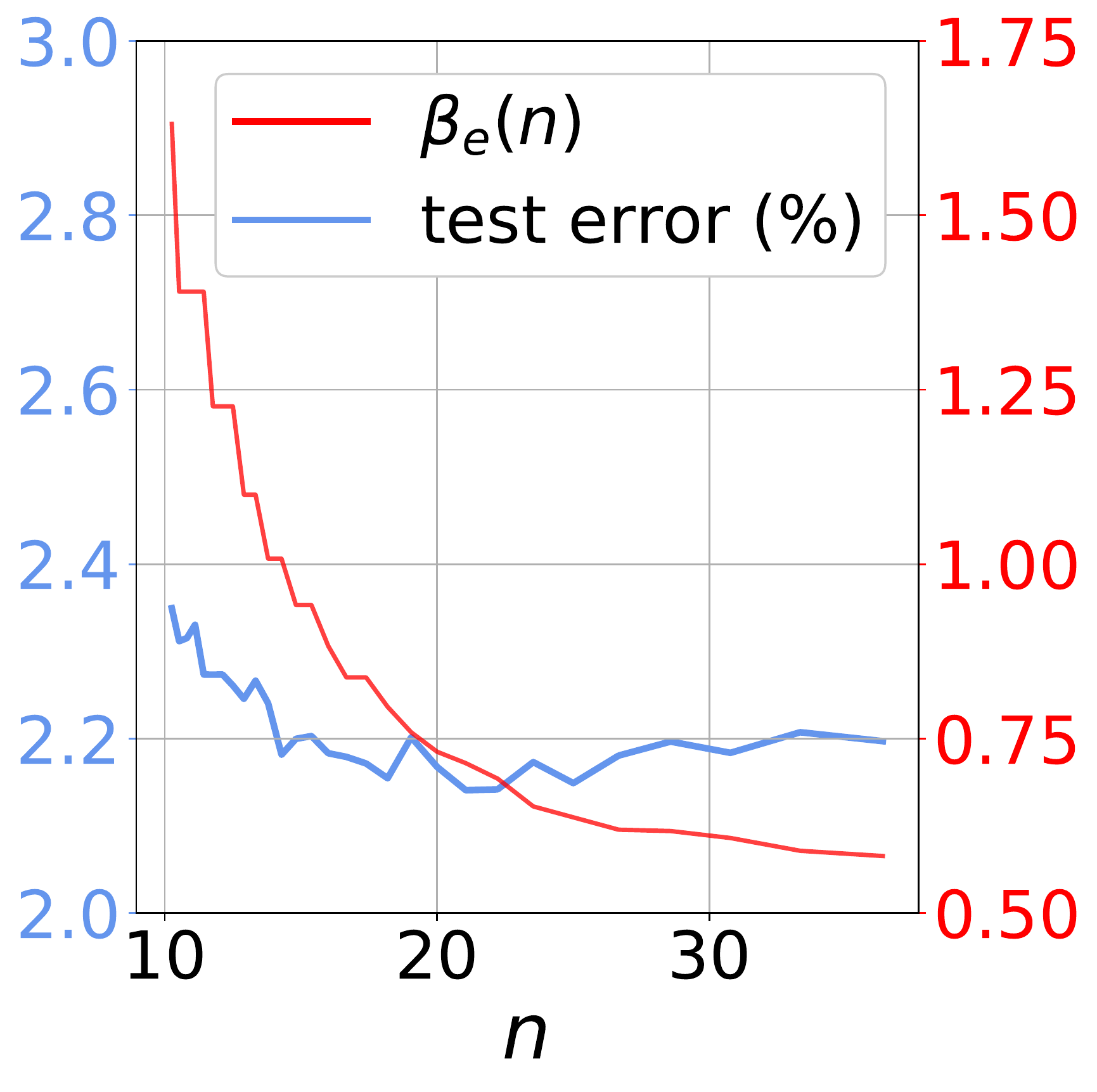}}\\ 
  \caption{\textbf{Decoupling capacity and variance}. The error (blue) is highly correlated with $Var(\mathcal{K})$, and less sensitive to $\beta_e$. (a) and (b) show the theoretical variance of the model correlates well with accuracy. (c) and (d) show the corresponding number of parameters $\beta_e$. Decreasing the variance (a) improves performance when $\beta_e$ is fixed (c). Increasing $\beta_e$ significantly (d) without reducing the variance (b) can cause degradation in performance due to overfitting.
  } 
  \label{fig:primaldualthm}
\end{figure*}
\subsection{Aggregated Residual Transformations}\label{sec:resnext_exps}

ResNeXt~\cite{resnext} introduces the concept of aggregated residual transformations utilizing group convolutions, which achieves better parameter/capacity trade off than its ResNet counterpart. In this paper, we view ResNeXt as a special case of CE, where the ensembling happens at the block level. We hypothesize that the optimal blocks identified with CE will lead to an overall better model when stacked up, and by doing so we get the benefit of factorizing the design space of a network to modular levels.
See Algorithm~\ref{alg:find_optimum} for the detailed recipe.

\begin{wrapfigure}[16]{r}{0.18\textwidth}
\centering
\begin{tikzpicture}[->,>=stealth',auto,node distance=0.1cm,
  thick,main node/.style={draw,font=\sffamily\Large\bfseries}]
  
  \tikzstyle{cell} = [rectangle, minimum width=1.7cm, minimum height=.5cm, text centered, draw=black]
\tikzstyle{roundcell} = [circle, text centered, draw=black, minimum size=.5cm]

\tikzstyle{blankcell} = [rectangle, minimum width=1.5cm, minimum height=0cm, text centered, draw=white]

    \node[blankcell] (N0) {};
    \node[cell] (N1) [below=.8cm of N0]{$1{\times} 1,  {\color{violet}n}$};
    \node[cell] (N2) [below=.3cm of N1]{$3{\times} 3, {\color{violet}n}$};
    \node[cell] (N3) [below=.3cm of N2]{$1{\times} 1, \text{n}_{\text{out}}$};
    \node[roundcell] (N4) [below=.3cm of N3] {$+$};
    \node[blankcell] (N5) [below=.3cm of N4]{};

    \draw [thin] (N0.south) -- node[pos=.15] {$\text{n}_{\text{in}}\text{-d in}$}  (N1.north); 
    \draw [thin] ($(N0)!0.4!(N1)$) to[out=10,in=15,looseness=1.2] (N4.east);
    \draw [thin] (N1.south) -- (N2.north); 
    \draw [thin] (N1.south) -- (N2.north); 
    \draw [thin] (N2.south) -- (N3.north); 
    \draw [thin] (N3.south) -- (N4.north); 
    \draw [thin] (N4.south) -- node[pos=.35] {$\text{n}_{\text{out}}\text{-d out}$}  (N5.north);

\end{tikzpicture}  \caption{}
  \label{fig:resnet_block}\end{wrapfigure}
  
For these experiments, we use both the \textbf{CIFAR-10/100} and the \textbf{ImageNet} datasets following the implementation details described in~\cite{resnext}. We also report results on \textbf{ImageNet}$\boldsymbol{64{\times}64}$ and \textbf{ImageNet}$\boldsymbol{32{\times}32}$, datasets introduced in~\cite{small-imagenet} that preserve the number of samples and classes of the original ImageNet-1K~\cite{imagenet}, but downsample the image resolutions to $64{\times}64$ and $32{\times}32$ respectively (see Appendix~\ref{sec:imagenet_64x64}). 

\textbf{Fitting $\alpha$ to a ResNet block.}~The first step in the optimization required for both the primal and dual objectives, is to approximate the $\alpha$ parameter in Eq.~\ref{eq:var0}. For convolutional layers, the coefficient multiplying $\alpha$ becomes $\sum_{l=1}^L\text{fan-in}_l^{-1}$ where $\text{fan-in}_l$ is the fan-in of layer $l$. Following Algorithm~\ref{alg:fit_alpha}, we approximate the $\alpha$ corresponding to a ResNet block parametrized by $\bold{n}=[n, n]^\top$ as depicted in Fig.~\ref{fig:resnet_block}. We compute a Monte Carlo estimate of the second moment of one diagonal entry of the NTK matrix for increasing width $n\in\llbracket1,256\rrbracket$ and fixed $\text{n}_{\text{in}}{=}\text{n}_{\text{out}}{=}256$. For simplicity, we fit the second moment normalized by the squared first moment, given by $e^{\alpha\sum_{l=1}^L\text{fan-in}_l^{-1}}$, which can easily be fitted with a first degree polynomial when considering its natural logarithm. We find $\alpha\approx 1.60$ and show the fitted second moment in Appendix Fig.~\ref{fig:alpha_fitting_resnet_block}.

\subsubsection{CIFAR-10/100}\label{cifar-exps}

\textbf{Primal formulation.}~As a baseline architecture, we use a $1{\times}128\text{d}$ ResNet, following the notations of~\cite{resnext} section 5.1. Following Algorithm~\ref{alg:find_optimum}, we compute $m(n)$ for $n\in \llbracket 1, 128 \rrbracket$ and find the optimum $n^p=10$ and $m^p\approx 37$, after adjusting $m^p$ to match the number of parameters of the baseline and account for rounding errors and different block topology approximations. As can be seen in Table~\ref{table:primal_resnext}, the model achieving the primal optimum, $37{\times}10\text{d}$, attains better test error on CIFAR-10/100 than the ResNeXt baseline $3{\times}64\text{d}$ at a similar parameter budget. We also report results for a wider baseline $8{\times}64\text{d}$ from~\cite{resnext} and show similar trends. The test error for multiple models sitting on the primal curve is depicted in Fig.~\ref{fig:resnext_primal_cifar100} for CIFAR-100 and Appendix Fig.~\ref{fig:resnext_primal_cifar10} for CIFAR-10. Test errors are averaged over the last 10 epochs over 10 runs.

\begin{figure}[!b]

\begin{minipage}[t]{0.485\textwidth}
\vspace{0pt}
\begin{algorithm}[H]
\KwInput{Baseline module with $\bold{n}{=}\{n_l\}_{l=1}^L$, a set of width ratios $\{r_j\}_{j=1}^R$, $T$, samples $\{x_1, x_2\}$.}
\KwOutput{Fitted $\alpha$.}
\For{$j=1,...,R$}
{
Construct module $\bold{n}_j=\{r_j\times{n_l}\}_{l=1}^L$.

\For{$t=1,...,T$}{
Sample weights of $\bold{n}_j$.

Compute $\mathcal{K}_{\bold{\tilde{n}_j}}(x_1,x_2)$.
}

Estimate $Var\big(\mathcal{K}_{\bold{\tilde{n}}}\big)$.
}

Fit $\alpha$ using Eq.~\ref{eq:var0}.

\caption{Fitting $\alpha$ per architecture}
\label{alg:fit_alpha}
\end{algorithm}

\end{minipage}
\hfill
\begin{minipage}[t]{0.491\textwidth}
\vspace{0pt}
\begin{algorithm}[H]
\KwInput{Baseline module with $\bold{\tilde{n}}{=}\{n_l\}_{l=1}^L$, a set of width ratios $\{r_j\}_{j=1}^R$, $T$, samples $\{x_1, x_2\}$, $\beta_s$.}
\KwOutput{optimal $\bold{n}^\star, m(\bold{n}^\star)$.}
Fit $\alpha$ using Algorithm~\ref{alg:fit_alpha}.

\uIf{Primal}{ find $\bold{n}=\bold{n}^p$ and $m=m^p$ using Eq.~\ref{primal}.}

\ElseIf{Dual}{find $\bold{n}=\bold{n}^d$ and $m=m^d$ using Eq.~\ref{dual}.}

Correct $\bold{n}$ and $m$ to nearest integer values.
 \caption{Find Optimal CE module}
 \label{alg:find_optimum}
\end{algorithm}

\end{minipage}
\end{figure}

\begin{figure*}[ht!]
\centering
  \subcaptionbox{Primal\label{fig:resnext_primal_cifar100}}{\includegraphics[height=1.85in]{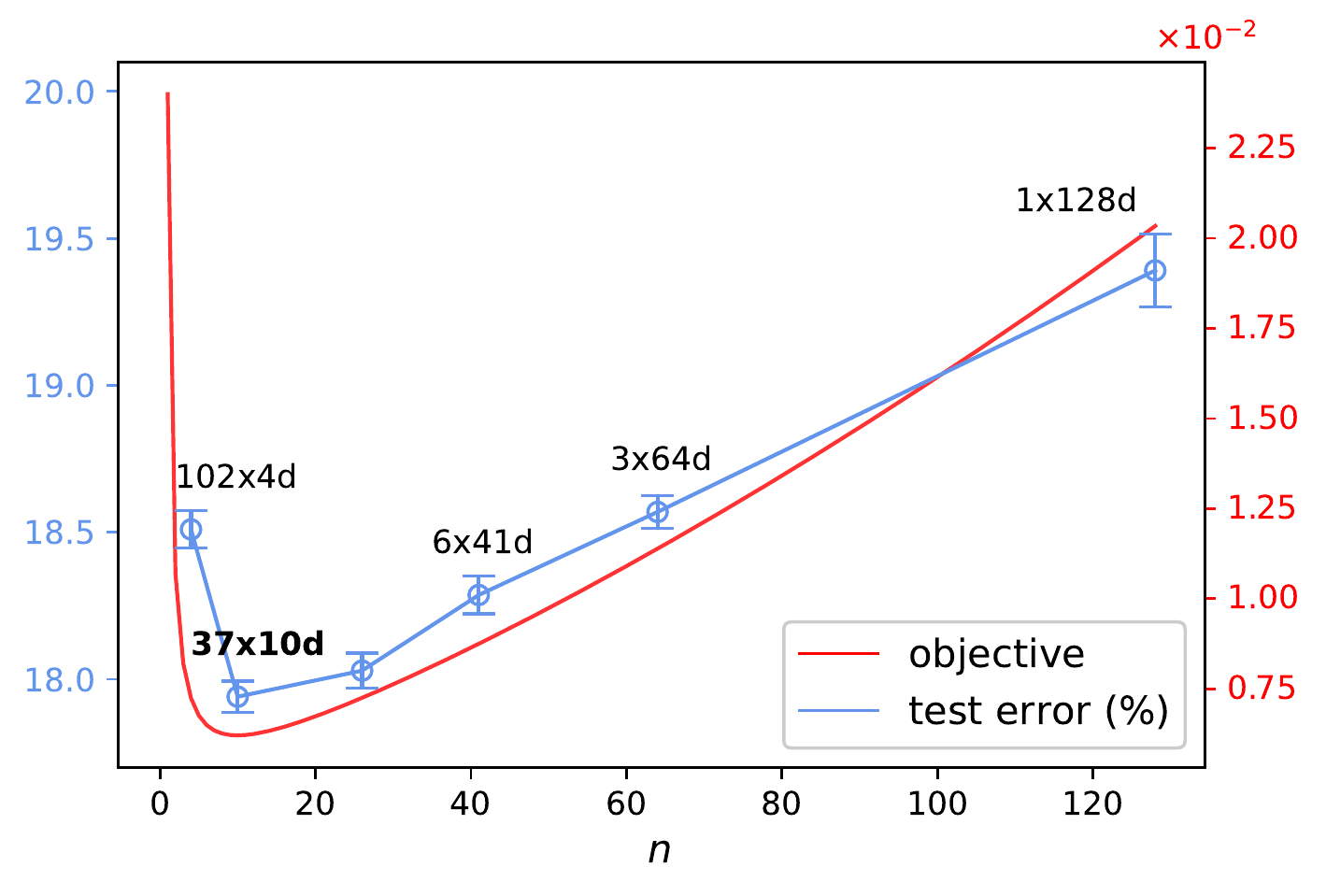}}\hfill 
  \subcaptionbox{Dual\label{fig:resnext_dual_cifar100}}{\includegraphics[height=1.77in]{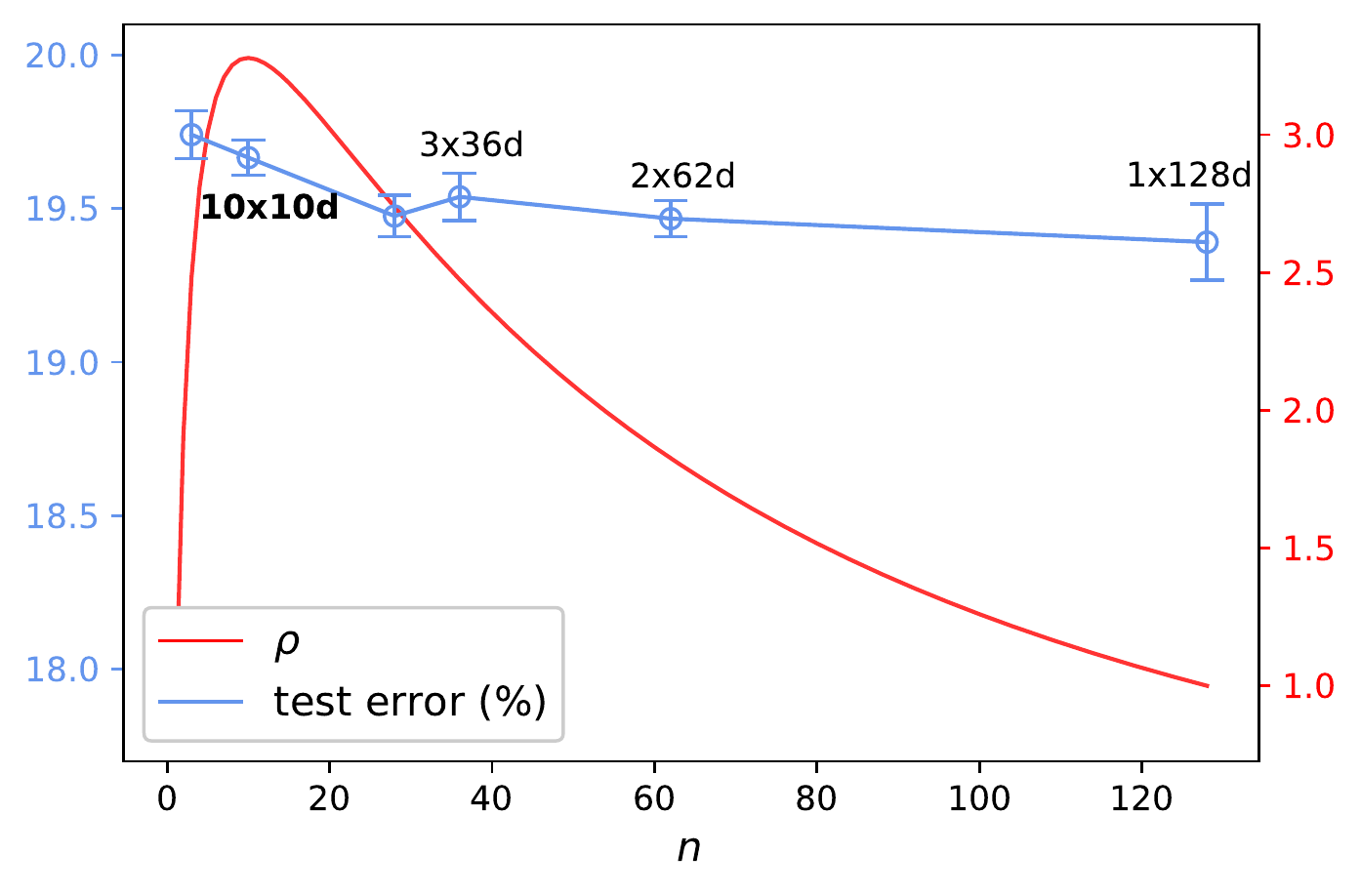}}\\
  \caption{\textbf{Test errors (in \%) on CIFAR-100.} Clear correlation is observed between test error and the primal curve. (a) Different models sitting on the primal curve. The model achieving the primal minimum, $37{\times}10\text{d}$, achieves the best error. (b) Different models sitting on the dual curve. The model $10{\times}10\text{d}$ achieves the dual maximum ($\rho\approx 3.3$) and a test error comparable to the baseline with $3.3$ times fewer parameters. Results are reported over 10 runs and shown with standard error bars.}
\end{figure*}

\textbf{Dual formulation.}~Using the same ResNet base block as for the primal experiment, thus using the same fitted $\alpha$, we compute the optimal $n^d$ and $m^d$ maximizing the parameter efficiency curve $\rho$ and find the same $n$ as the primal, $n^d=10$, and $m^d\approx 10$. The resulting ResNeXt network has $3.3$ times fewer parameters than the baseline and achieves similar or slightly degraded performance on CIFAR-10/100 as shown in Table~\ref{table:dual_resnext}. The efficiency curve $\rho$ depicted in red in Fig.~\ref{fig:resnext_dual_cifar100} is constructed using a single ResNet block topology and with non integer numbers for $m$ as described above. Thus it only approximates the \emph{real} parameter efficiency, explaining why some models in the close vicinity of the optimum have a slightly higher real efficiency as can be seen in Table~\ref{table:dual_resnext}. The test error for multiple models sitting on the dual curve is depicted in Fig.~\ref{fig:resnext_dual_cifar100} for CIFAR-100 and Appendix Fig.~\ref{fig:resnext_dual_cifar10} for CIFAR-10.

\subsubsection{ImageNet}\label{appendix:imagenet}

\textbf{Primal formulation.}~Following~\cite{resnext}, we use ResNet-50 and ResNet-101 as baselines and report results in Table~\ref{tab:imagenet_primal}. Our ResNet-50 based optimal model, $12{\times}10\text{d}$, obtains slightly better top-1 and top-5 errors than the baseline $32{\times}4\text{d}$ reported in~\cite{resnext}. This is quite remarkable given that~\cite{resnext} converged to this architecture via an expensive grid search. Our ResNet-101 based optimal model achieves a significantly better top-1 and top-5 error than the ResNet-101 baseline, and a slightly higher top-1 and top-5 error than the ResNeXt baseline $32{\times}4\text{d}$.

\textbf{Dual formulation.} Using ResNet-50 and ResNet-101 as baselines, we find models that achieve similar top-1 and top-5 errors with significantly less parameters. Detailed results can be found in Table~\ref{tab:imagenet_dual}.

\begin{table}[!h]
\centering
\vspace{0pt}
\renewcommand{\arraystretch}{1.1}
\captionsetup{aboveskip=0pt}
\begin{subtable}[t]{.46\linewidth}
\centering
\vspace{1pt}
\begin{tabular}{l|p{.8cm}p{.6cm}p{.6cm}}
\hline
 Model & Params& {C10} & {C100} \\ \hline
$1{\times}128\text{d}$~\cite{resnext}\textsuperscript{\textdagger}  & 13.8M &  4.08    & 19.39           \\
$3{\times}64\text{d}$~\cite{resnext}\textsuperscript{\textdagger} & 13.3M &   3.96   & 18.57       \\
$37{\times}10\text{d}$ (Ours) & 13.7M & 3.82 & \textbf{17.94}\\
$28{\times}12\text{d}$ (Ours)\textsuperscript{\ddag}  & 12.9M & \textbf{3.74} & 18.05\\
\hline
Wide ResNet~\cite{wideresnet}  & 36.5M &  4.17    & 20.50           \\
$1{\times}226\text{d}$ & 36.3M &   3.88    & 18.36 \\

$8{\times}64\text{d}$~\cite{resnext} & 34.4M &  \textbf{3.65}    & 17.77 \\
$101{\times}10\text{d}$ (Ours) & 36.3M &   \textbf{3.65}   & \textbf{17.44}       \\
\hline
\end{tabular}
\caption{Primal}
\label{table:primal_resnext}
\end{subtable}
\hspace{-.1cm}
\begin{subtable}[t]{.5\textwidth}
\centering
\vspace{0pt}
\flushright
\begin{tabular}{l|cp{1cm}p{.6cm}p{.6cm}}

\hline
Model & $\rho$ & Params & {C10} & {C100} \\ \hline
$1{\times}128\text{d}$~\cite{resnext}\textsuperscript{\textdagger}\textsuperscript{\textsection}  & 1 & 13.8M &  4.75   & 20.74\\
$10{\times}10\text{d}$ (Ours)\textsuperscript{\textsection} & \textbf{3.3} & \textbf{4.22M} & 4.70 & 20.84\\
\hline
$1{\times}128\text{d}$~\cite{resnext}\textsuperscript{\textdagger}  & 1 &  13.8M &  4.08   & 19.39       \\
$10{\times}10\text{d}$ (Ours) & 3.3 & 4.22M & 4.21 & 19.67\\ 
$8{\times}12\text{d}$ (Ours)\textsuperscript{\ddag}  & \textbf{3.3} & \textbf{4.19M} & 4.20 & 19.58\\ 
\hline
Wide ResNet~\cite{wideresnet} & 1 & 36.5M &  4.17 & 20.50           \\
$1{\times}226\text{d}$ & 1 & 36.3M & 3.88 & 18.36 \\
$2{\times}64\text{d}$~\cite{resnext} &  4.0 & 9.13M  & 4.02  & ~~~- \\ 
$14{\times}10\text{d}$ (Ours) &  \textbf{6.5} & \textbf{5.63M} & 4.01  & 19.04 \\ 
$6{\times}25\text{d}$ (Ours) &  5.0 & 7.26M & 3.94  & 18.92 \\
$3{\times}58\text{d}$ (Ours) &  3.1 & 11.6M & 3.87  & 18.75 \\
$2{\times}98\text{d}$ (Ours) &  2.1 & 17.4M & 3.99  & 18.32 \\

\hline
\end{tabular}
\caption{Dual}
\label{table:dual_resnext}
\end{subtable}
  \caption{\textbf{Primal and dual results for ResNeXt-29 baselines on CIFAR-10/100}. Test errors (in \%) for CIFAR-10 (C10) and CIFAR-100 (C100) along with model sizes are reported. All models are variants of ResNeXt-29 except for Wide ResNet. (a) The optimally smooth models, $37{\times}10\text{d}$ and $101{\times}10\text{d}$, surpass the baselines with the same number of parameters. (b) The optimally compact models only use a fraction of the parameters, yet attain similar or slightly degraded test errors. $\rho$ indicates the parameter efficiency. \textsuperscript{\dag} indicates we reproduced results on baseline architectures from the cited paper, \textsuperscript{\ddag} indicates models in the close vicinity of the optima.  Results are averaged over 10 runs. Models were trained on 8 GPUs unless indicated by \textsuperscript{\textsection}, in which case they were trained on a single GPU.}
\end{table}

\begin{table}[h!]

\centering
\vspace{0pt}
\renewcommand{\arraystretch}{1.1}
\captionsetup{aboveskip=0pt}

\begin{subtable}{\linewidth}
\centering
\begin{tabular}{l|cccc}
\hline
       & Params & Top-1 error  & Top-5 error\\ \hline
ResNet-50, $1{\times}64\text{d}$~\cite{resnext}\textsuperscript{\textdagger}  &  25.6M &  23.93   & 7.11 \\
ResNeXt-50, $32{\times}4\text{d}$~\cite{resnext}\textsuperscript{\textdagger}  &  25.0M &  22.42   & 6.36 \\
ResNeXt-50, $12{\times}10\text{d}$ (Ours)  &  25.8M &  \textbf{22.37}   & \textbf{6.30} \\
ResNeXt-50, $15{\times}8\text{d}$ (Ours)\textsuperscript{\ddag}  &  25.1M &  22.39   & 6.38 \\

\hline
ResNet-101, $1{\times}64\text{d}$~\cite{resnext}\textsuperscript{\textdagger}  & 44.5M &  22.32   & 6.25 \\
ResNeXt-101, $32{\times}4\text{d}$~\cite{resnext}\textsuperscript{\textdagger}  & 44.2M &  \textbf{21.01}   & \textbf{5.72} \\
ResNeXt-101, $12{\times}10\text{d}$ (Ours)  &  45.5M &  21.16   & 5.74 \\
ResNeXt-101, $15{\times}8\text{d}$ (Ours)\textsuperscript{\ddag}  &  44.2M &  21.20   & 5.76 \\

\hline
\end{tabular}
\caption{Primal}
\label{tab:imagenet_primal}
\end{subtable}

\begin{subtable}{\linewidth}
\centering
\begin{tabular}{l|cccc}
\hline
      & $\rho$ & Params & Top-1 error  & Top-5 error\\ \hline
ResNet-50, $1{\times}64\text{d}$~\cite{resnext}\textsuperscript{\textdagger}  & 1 &  25.6M &  23.93   & \textbf{7.11} \\

ResNeXt-50, $3{\times}23\text{d}$ (Ours)  & 1.3 &  19.4M &  \textbf{23.80} &  \textbf{7.11} \\
ResNeXt-50, $4{\times}16\text{d}$ (Ours)  & \textbf{1.5} &  \textbf{17.1M} &  24.00 &  \textbf{7.11} \\
\hline
ResNet-101, $1{\times}64\text{d}$~\cite{resnext}\textsuperscript{\textdagger}  & 1 &  44.5M &  22.32   & 6.25 \\
ResNeXt-101, $3{\times}23\text{d}$ (Ours)  & 1.4 &  32.9M &  \textbf{22.06}   & \textbf{6.11} \\
ResNeXt-101, $5{\times}12\text{d}$ (Ours)  & \textbf{1.7} &  \textbf{25.8M} &  22.30   & 6.27 \\
\hline

\end{tabular}
\caption{Dual}
\label{tab:imagenet_dual}
\end{subtable}

\caption{\textbf{Primal and dual results for ResNet baselines on ImageNet.} Top-1 and top-5 errors (in \%) and model sizes are reported. $\rho$ indicates the parameter efficiency, \textsuperscript{\dag} indicates we reproduced results on baseline architectures from the cited paper, \textsuperscript{\ddag} indicates models in the close vicinity of the optimum. Results are averaged over 5 runs.}
\label{table:imagenet}
\end{table}

\textbf{Implementation details.}~We follow~\cite{resnext} for the implementation details of ResNet-50, ResNet-101 and their ResNeXt counterparts. We use SGD with $0.9$ momentum and a batch size of 256 on $8$ GPUs ($32$ per GPU). The weight decay is $0.0001$ and the initial learning rate $0.1$. We train the models for $100$ epochs and divide the learning rate by a factor of $10$ at epoch $30$, $60$ and $90$. We use the same data normalization and augmentations as in~\cite{resnext} except for lighting that we do not use.





\section{Related Work}
Various forms of multi-pathway neural architectures have surfaced over the years. In the seminal AlexNet architecture \cite{alexnet}, group convolutions were used as a method to distribute training on multiple GPUs. More recently, group convolutions were popularized by ResNeXt \cite{resnext}, empirically demonstrating the benefit of aggregating multiple residual branches. In \cite{shufflenet}, a channel shuffle unit was introduced in order to promote information transfer between different groups. In \cite{condensenet} and \cite {other}, the connections between pre-defined set of groups are learned in a differentiable manner, and in \cite{prune}, grouping is achieved through pruning of full convolutional blocks. In a seemingly unrelated front, the theoretical study  of wide neural networks has seen considerable progress recently. A number of papers \cite{wide,yang,exact,guy,enhanced} have followed in the footsteps of the original NTK paper \cite{NTK}. In \cite{wide}, it is shown that wide models of any architecture evolve as linear models, and in \cite{yang}, a general framework for computing the NTK of a broad family of architectures is proposed. Finite width corrections to the NTK are derived in \cite{boris1,me,guy}. In this work, we extend the "wide" regime to the multiplicity dimension, and show two distinct regimes where different kernels reign. We then use finite corrections of NTK to formulate two optimality criterions, and demonstrate their usefulness in predicting efficient and performant ensembles.\\ 
\section{Conclusion}
Understanding the effects of model architecture on training and test performance is a longstanding goal in the deep learning community. In this work we analyzed collegial ensembling, a general technique used in practice where multiple and functionally identical pathways are aggregated. We showed that collegial ensembles exhibit two distinct regimes of over-parameterization, defined by large width and large multiplicity, with two distinct kernels governing the dynamics of each. In between these two regimes, we introduced a framework for deriving optimal ensembles in a sense of minimum capacity or maximum trainability. Empirical results on practical models demonstrate the predictive power of our framework, paving the way for more principled architecture search algorithms.

\bibliographystyle{plain}
\bibliography{refs}

\appendix
\newpage

\section*{Appendix}
\section{Results and Implementation Details: Downsampled ImageNet}\label{sec:imagenet_64x64}

\subsection{Experiments on ImageNet$\bold{32{\times}32}$}
\textbf{Implementation details.}~
We use the same ResNeXt-29 architectures from the CIFAR experiments. We use SGD with $0.9$ momentum and a batch size of $1024$ on $8$ GPUs ($128$ per GPU). The weight decay is $0.0005$ and the initial learning rate $0.08$. We train the models for $80$ epochs and divide the learning rate by a factor $5$ at epoch $20$, $40$ and $60$. We use the same data normalization and augmentations as in~\cite{small-imagenet}.

\textbf{Primal formulation.}~Using the same baseline architecture $1{\times}128\text{d}$ as for the CIFAR experiments, we train the model achieving the optimum, $37{\times}10\text{d}$, and report results in Table~\ref{table:imagenet}. Our optimal model achieves lower top-1 and top-5 errors than the baseline ResNeXt architecture $3{\times}64\text{d}$ derived in~\cite{resnext} at a similar parameter budget. We use the same augmentations and learning rate schedule as~\cite{small-imagenet}.

\textbf{Dual formulation.}~Using the same baseline $1{\times}128\text{d}$ and optimally compact architecture $10{\times}10\text{d}$ derived in section~\ref{cifar-exps}, we observe a similar trend: our optimal model suffers lighter top-1 and top-5 degradation than the Wide ResNet variant with a reduced parameter budget, with $2.8$ times fewer parameters than the baseline. Sampling models on the dual curve with lower $\rho$ such as $3{\times}{40}\text{d}$, we find models that suffer less than a percent drop in top-1 and top-5 error with a significantly lower parameter count.

\begin{table}[h]
\centering
\begin{tabular}{l|cccc}
\hline
      & $\rho$ & Params & Top-1 error  & Top-5 error\\ \hline
Wide ResNet 28-10~\cite{small-imagenet} & - & 37.1M &  40.96   & 18.87\\
ResNet-29, $1{\times}128\text{d}$\textsuperscript{\textdagger}  & 1 &  14.8M & 40.61   & 17.82\\
ResNeXt-29, $3{\times}64\text{d}$\textsuperscript{\textdagger} & 1 &  14.4M &  39.58   & 17.09\\

ResNeXt-29, $37{\times}10\text{d}$ (Ours) & 1 &  14.8M &  \textbf{38.41}   & \textbf{16.13}\\
\hline
Wide ResNet 28-5~\cite{small-imagenet} & 1.6 &  9.5M &  45.36  & 21.36\\
ResNeXt-29, $10{\times}10\text{d}$ (Ours) & \textbf{2.8} &  \textbf{5.2M} &  43.36   & 19.65\\
ResNeXt-29, $3{\times}40\text{d}$ (Ours) & 1.8 & 8.0M &  41.54   & 18.58\\

\hline
\end{tabular}
\caption{\textbf{Primal and dual results for ResNet baselines on ImageNet$\boldsymbol{32{\times}32}$}. Top-1 and top-5 errors (in \%) and model sizes are reported. The optimally smooth model, $37{\times}10\text{d}$, surpasses the baseline architectures from~\cite{resnext} (indicated with \textsuperscript{\textdagger}) with the same number of parameters. The optimally compact model, $10{\times}10\text{d}$, achieves slightly degraded results but with $2.8$ times fewer parameters. Results are averaged over 5 runs.}
\label{table:imagenet}
\end{table}

\subsection{Experiments on ImageNet$\bold{64{\times}64}$}
\textbf{Implementation details.}~In order to adapt the ResNeXt-29 architectures used for CIFAR-10/100 and ImageNet$32{\times}32$ to the resolution of ImageNet$64{\times}64$, we add an additional stack of three residual blocks following~\cite{small-imagenet}. Following the general parametrization of ResNeXt~\cite{resnext}, we multiply the width of this additional stack of blocks by $2$ and downsample the spatial maps by the same factor using a strided convolution in the first residual block. We use SGD with $0.9$ momentum and a batch size of $512$ on 8 GPUs ($64$ per GPU). The weight decay is $0.0005$ and the initial learning rate $0.04$. We train the models for $60$ epochs and divide the learning rate by a factor $5$ at epoch $20$, $30$ and $40$. We use the same data normalization and augmentations as~\cite{small-imagenet}.
\begin{table}[h!]
\centering
\begin{tabular}{l|cccc}
\hline
       & $\rho$ & Params & Top-1 error  & Top-5 error\\ \hline
Wide ResNet 36-5~\cite{small-imagenet} & - & 37.6M &  32.34   & 12.64\\
ResNet-38, $1{\times}96\text{d}$\textsuperscript{\textdagger}  & 1 &  37.5M &  29.36   & 10.10\\
ResNeXt-38, $2{\times}64\text{d}$\textsuperscript{\textdagger}  & 1 &  39.0M &  28.86   & 9.72\\
ResNeXt-38, $22{\times}10\text{d}$ (Ours)  & 1 &  36.3M &  \textbf{28.34}   & \textbf{9.38}\\
ResNeXt-38, $8{\times}10\text{d}$ (Ours)  & \textbf{2.3} &  \textbf{16.3M} &  30.93   & 11.02\\
\hline
ResNet-38, $1{\times}128\text{d}$\textsuperscript{\textdagger}  & 1 &  57.8M &  28.56   & 9.67\\
ResNeXt-38, $3{\times}64\text{d}$\textsuperscript{\textdagger}  & 1 &  56.1M &  28.01   & 9.28\\
ResNeXt-38, $37{\times}10\text{d}$ (Ours)  & 1 &  57.7M &  \textbf{27.24}   & \textbf{8.74}\\
ResNeXt-38, $10{\times}10\text{d}$ (Ours)  & \textbf{3.0} &  \textbf{19.1M} &  30.22   & 10.60\\
\hline

\end{tabular}
\caption{\textbf{Primal and dual results for ResNet baselines on ImageNet$\boldsymbol{64{\times}64}$}. Top-1 and top-5 errors (in \%) and model sizes are reported. The optimally smooth models, $22{\times}10\text{d}$ and $37{\times}10\text{d}$, surpass the baseline architectures from~\cite{resnext} (indicated with \textsuperscript{\textdagger}) with the same number of parameters. The optimally compact models, $8{\times}10\text{d}$ and $10{\times}10\text{d}$, achieve slightly degraded results but with significantly fewer parameters. Results are averaged over 5 runs.}
\label{table:imagenet64x64}
\end{table}

\newpage
\section{Results on CIFAR-10}

Results are shown in Fig.~\ref{fig:resnext_cifar10} and implementation details can be found in the main text in Sec.~\ref{cifar-exps}.
\begin{figure}[h!]
\centering
  \subcaptionbox{Primal\label{fig:resnext_primal_cifar10}}{\includegraphics[height=1.85in]{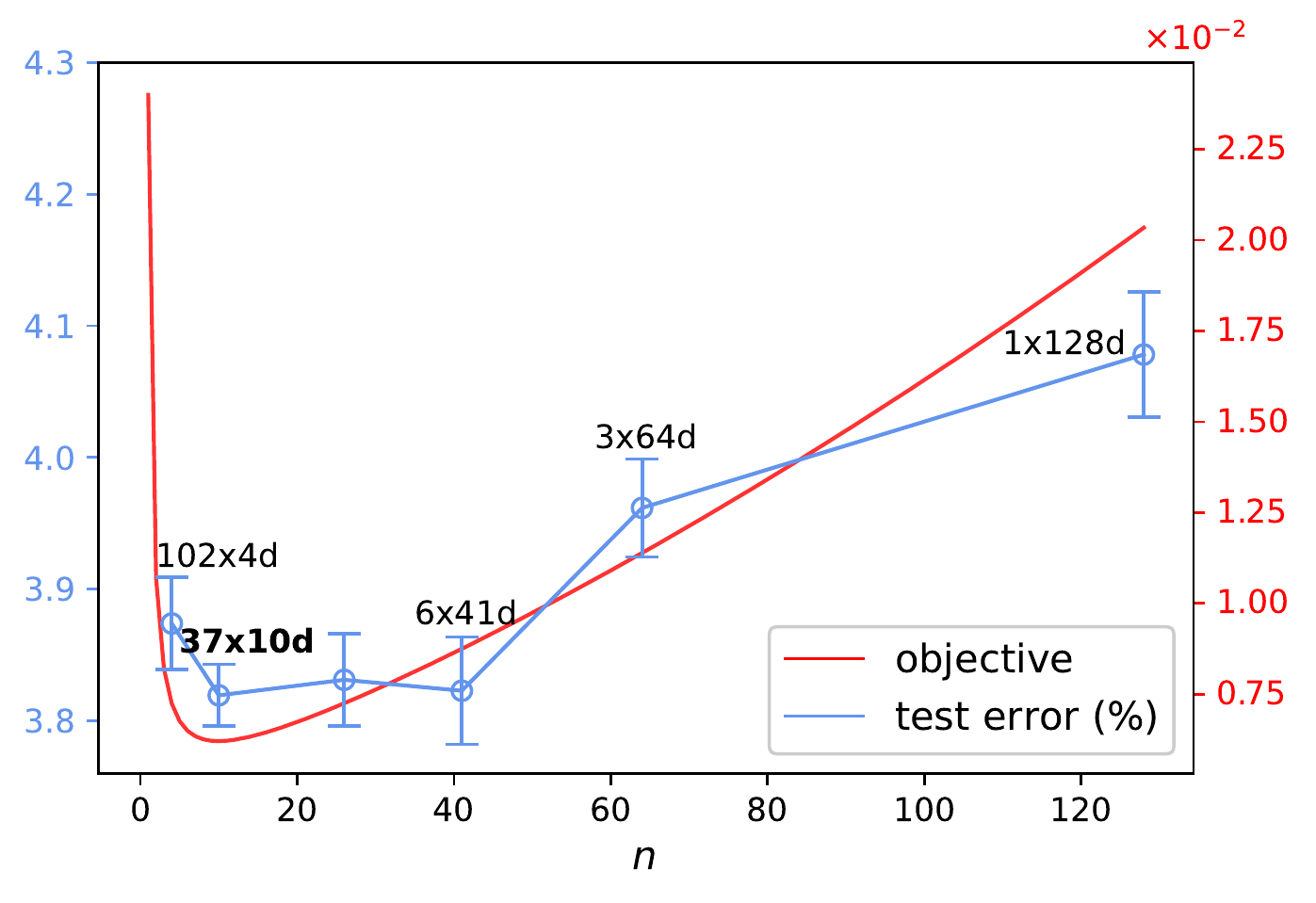}}\hfill 
  \subcaptionbox{Dual\label{fig:resnext_dual_cifar10}}{\includegraphics[height=1.79in]{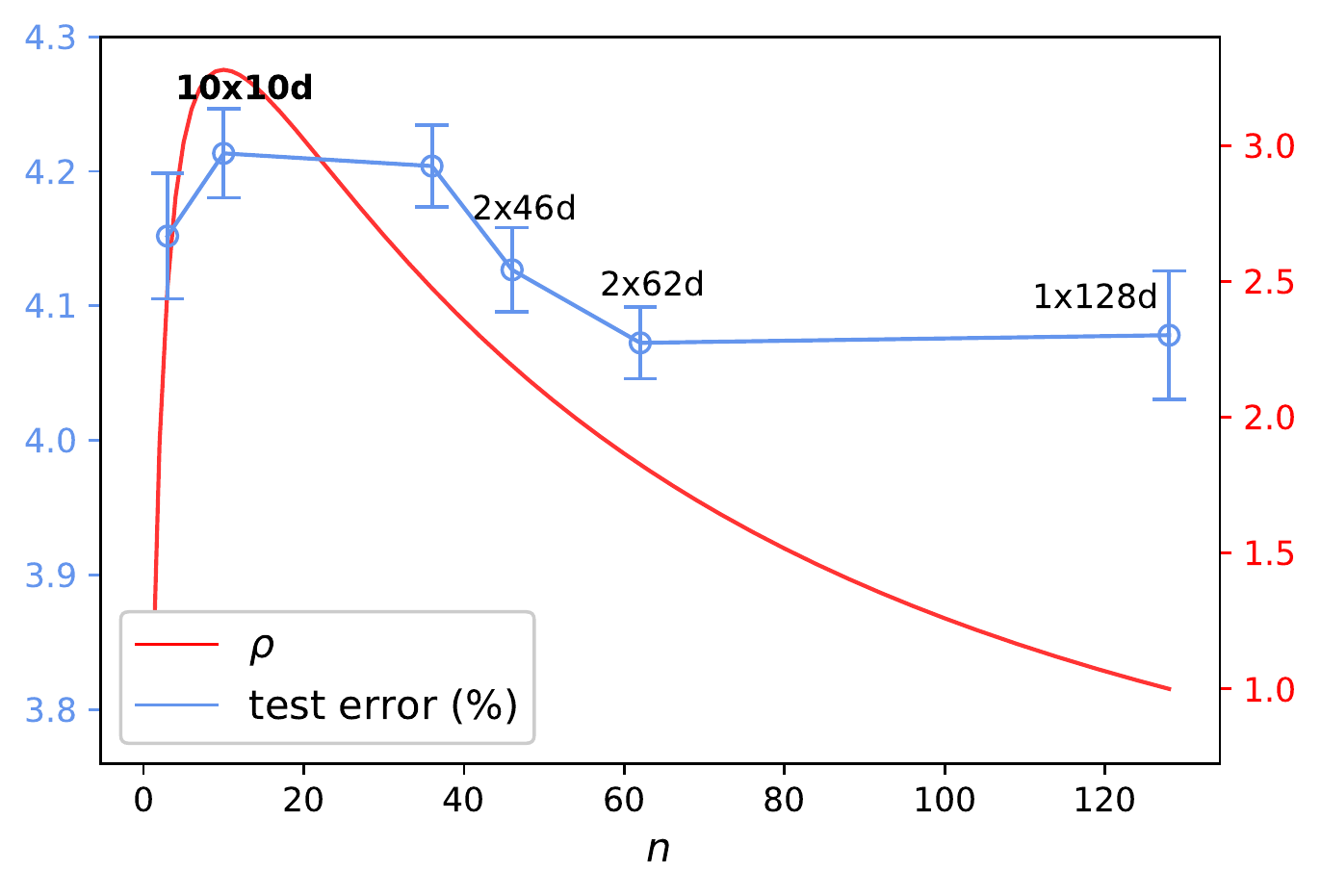}}\\
  \caption{\textbf{Test errors (in \%) on CIFAR-10.} Clear correlation is observed between test error and the primal curve. (a) Different models sitting on the primal curve. The model achieving the primal minimum, $37{\times}10\text{d}$, achieves the best error. (b) Different models sitting on the dual curve. The model $10{\times}10\text{d}$ achieves the dual maximum ($\rho\approx 3.3$) and a slightly higher test error than the baseline with $3.3$ times fewer parameters. Results are reported over 10 runs and shown with standard error bars.}
  \label{fig:resnext_cifar10}
\end{figure}

\section{FLOPs efficiency}\label{appendix:flops_efficiency}

In Sec.\ref{sec:efficient_ensembles} and the rest of the paper, we considered the parameter efficiency $\rho$ defined as the ratio between the number of parameters of the baseline model and the ensemble (see Eq.~\ref{eq:rho}). Using this definition of efficiency, models satisfying the primal objective were models with similar number of parameters. Instead of using the parameter efficiency, we can consider FLOPs efficiency in the same way:
\[
\rho^{\text{flop}}(m,\bold{n}) \triangleq \frac{\beta^{\text{flop}}_s}{\beta^{\text{flop}}_e(\bold{n})}=\frac{\beta^{\text{flop}}_s}{m\beta^{\text{flop}}(\bold{n})},
\]

where $\beta^{\text{flop}}_s$ and $\beta^{\text{flop}}_e$ are the number of FLOPs of the baseline model and the total number of FLOPs in the ensemble respectively. We report results for the primal formulation on CIFAR-10/100 in Table~\ref{table:cifar_flops_results}. We see that the model achieving the primal optimum, $44{\times}8\text{d}$, attains better test error on CIFAR-10 and CIFAR-100 with similar number of FLOPs.

\begin{table}[!h]

\centering
\begin{tabular}{l|cccc}
\hline
 Model & GFLOPs & Params& {C10} & {C100} \\ \hline
ResNet-29 $1{\times}128\text{d}$~\cite{resnext}\textsuperscript{\textdagger}  & 4.19 & 13.8M &  4.08    & 19.39           \\
ResNeXt-29 $3{\times}64\text{d}$~\cite{resnext}\textsuperscript{\textdagger} & 4.15 & 13.3M &   3.96   & 18.57       \\
ResNeXt-29 $44{\times}8\text{d}$ (Ours) & 4.20 & 12.7M & \textbf{3.66} & \textbf{17.86}\\
ResNeXt-29 $60{\times}6\text{d}$ (Ours)\textsuperscript{\ddag} & 4.17 & 12.6M & 3.73 & 18.04\\

\hline
\end{tabular}
\caption{\textbf{Results for ResNeXt-29 baselines on CIFAR-10/100 when keeping FLOPs constant instead of \# parameters}. Test errors (in \%) for CIFAR-10 (C10) and CIFAR-100 (C100) along with model GFLOPs and sizes are reported. The optimally smooth model, $44{\times}8\text{d}$, surpasses the baselines with the same number of FLOPs. \textsuperscript{\dag} indicates we reproduced results on baseline architectures from the cited paper, \textsuperscript{\ddag} indicates models in the close vicinity of the optimum. Results are averaged over 10 runs.}
\label{table:cifar_flops_results}
\end{table}

\section{Fitting $\alpha$ to a ResNet Block}
\begin{figure}[H]
\hspace*{-1.2cm}
\centering
\includegraphics[width=2.95in]{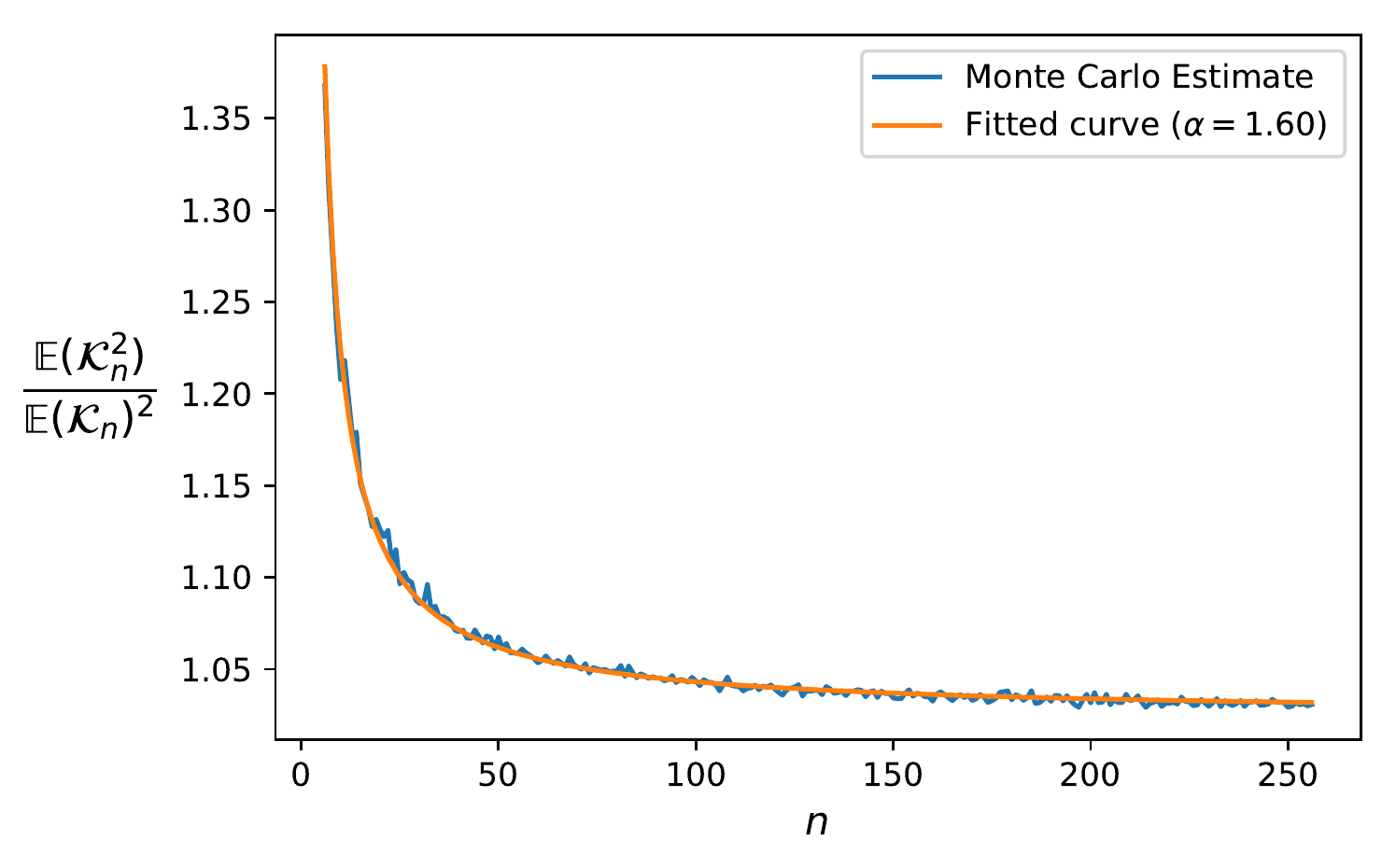}
  \caption{\textbf{Estimating the variance of a ResNet block and fitting} $\boldsymbol{\alpha}$. The Monte Carlo estimate is calculated over $2000$ trials and $\alpha$ is fitted following Algorithm~\ref{alg:fit_alpha} as described in Sec.~\ref{sec:resnext_exps}.}
\label{fig:alpha_fitting_resnet_block}
\end{figure}

\section{Figure Illustrating Theorem 1}

\begin{figure}[H]
\centering
  \includegraphics[height=2.0in]{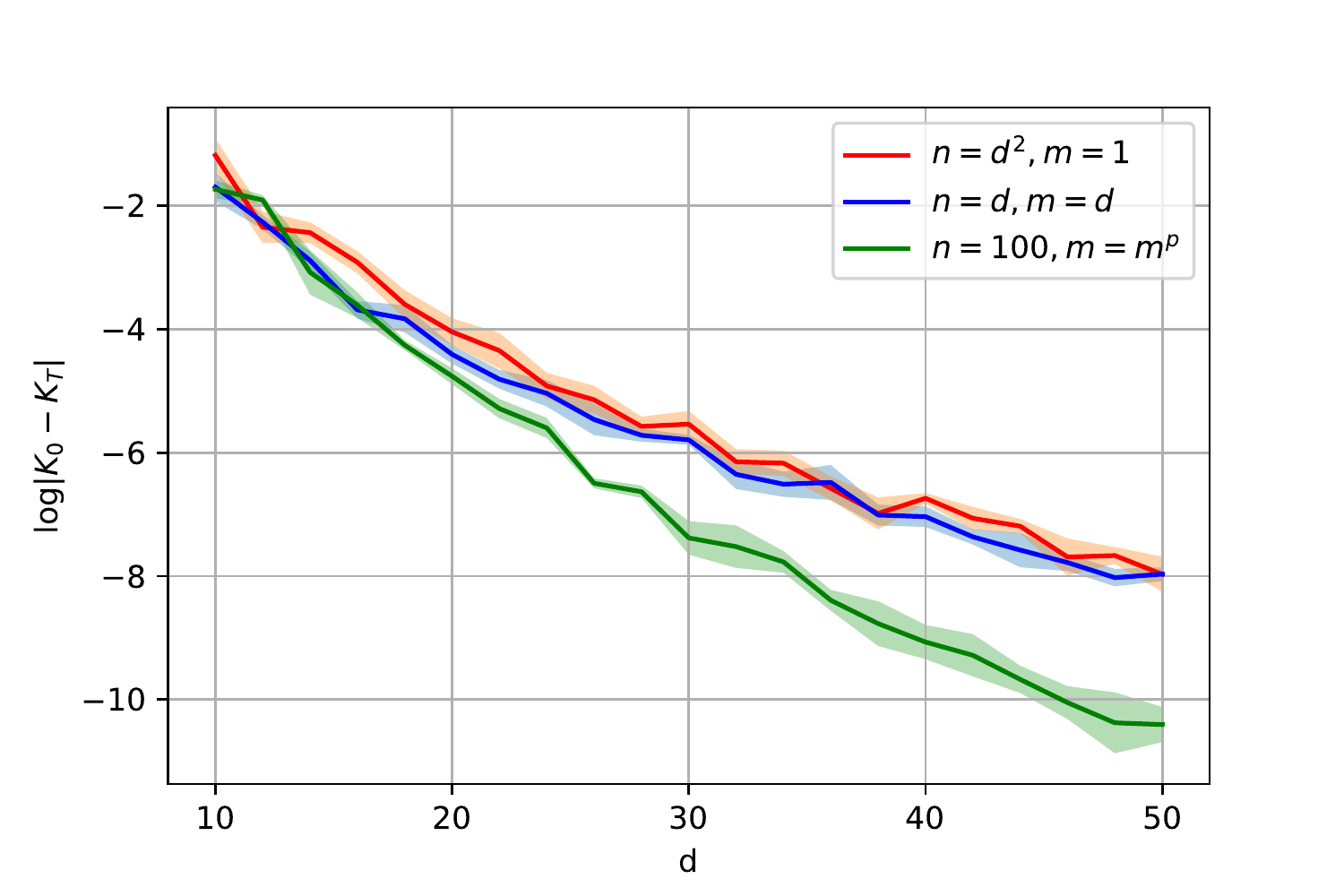}\hspace{1em}%
  \caption{Dynamics of the NTK during training as a function of width $n$ and multiplicity $m$ for (a) baseline single model (b) ensemble (c) ensemble with a constant width $n = 100$, and multiplicity $m$ to match the number of parameters in the baseline (red). The NTK was computed for a single off-diagonal entry for a depth $L = 4$ fully connected network trained on MNIST. The y axis corresponds to the absolute change in log scale between the NTK value at initialization, and after training for $T=100$ epochs. As predicted in Theorem~\ref{thm:evolution}, the baseline model with $m=1,n=d^2$ and the ensemble with $m = n = d$ have equal $mn$, therefore exhibit similar correction of the NTK. In (c), the change of the NTK becomes smaller than the baseline, as mn is considerably larger, although the total parameter count is the same as the baseline.}
  \label{fig:dynamics}
\end{figure}

\section{Proofs of Lemma 1 and Theorem 1}\label{proofs}
\ensemblea*
\begin{proof}
Recall that the NTK of the ensemble is given by the mean:
\[
\mathcal{K}^e(X;\Theta) = \frac{1}{m}\sum_{j=1}^m \mathcal{K}_{\bold{n}}(X;\theta_j). 
\]
Note that expectation of each member in the average is identical and finite under Lebesgue integration:
\[
\mathcal{K}_\bold{n}^\infty(X) = \E_\theta[\mathcal{K}_{\bold{n}}(X;\theta_j)].
\]
Since each member of the ensemble is sampled independently, we have from the strong law of large numbers (LLN):
\[
\frac{1}{m}\sum_{j=1}^m \mathcal{K}_{\bold{n}}(X;\theta_j) \stackrel{a.s}{\longrightarrow} \mathcal{K}_\bold{n}^\infty(X)~~~\textnormal{ as $m \to \infty$}.
\]
Proving the claim.
\end{proof}

\ensembleb*
\begin{proof}
In the following proof, we assume for the sake of clarity our training data contains a single example, so that $\mathcal{K}_n, \mathcal{K}^e,\mathcal{F}_n,\mathcal{F}^e \in \mathcal{R}$. The results however hold in the general case. Throughout the proof, we use $\Theta_t$ to denote the weights at time $t$, while $\Theta,\theta_{j}$ denote the weights at initialization.
    

For analytic activation functions, the time evolved kernel $\mathcal{K}^e(\Theta_t)$ is analytic with respect to $t$. Therefore, at any time $t$ we may approximate the kernel using Taylor expansion evaluated at $t=0$:
\[
\mathcal{K}^e(\Theta_t) - \mathcal{K}^e(\Theta) = \sum_{r=1}^\infty \frac{\partial ^r \mathcal{K}^e(\Theta) }{\partial t ^r}\frac{t^r}{r!}.
\]
Similarly to the technique used in \cite{wide}, we assume we may exchange the Taylor expansion with large width and multiplicity limits.  
We now analyze each term in the Taylor expansion separately. Using the ensemble NTK definition in Eq.~\ref{eq:ensemble_ntk} in the main text, the $r$-th order term of the ensemble NTK is given by:
\[\label{rth}
\frac{\partial ^r \mathcal{K}^e(\Theta) }{\partial t ^r} = \frac{1}{m}\sum_{j=1}^m \frac{\partial ^r \mathcal{K}_n(\theta_{j}) }{\partial t ^r} = \frac{1}{m}\sum_{j=1}^m\Big(\frac{\partial}{\partial t} \Big)^r\mathcal{K}_n(\theta_{j}).
\]
Next we derive the time derivative operator $\frac{\partial}{\partial t}$, under gradient flow and $L_2$ loss.
Given label $y$, we can denote the residual terms $\mathcal{R}_j = \big(\mathcal{F}_n(\theta_{j}) - \frac{y}{\sqrt{m}}\big) \in \mathcal{R}^N$ and the total model residual as $\mathcal{R}(\Theta) = \frac{1}{\sqrt{m}}\sum_{j=1}^m\mathcal{R}_j$, such that the $L_2$ cost given by $\mathcal{L} = \frac{1}{2}\big(\mathcal{R}(\Theta)\big)^2$.The model parameters in this case evolve according to:
\[
\dot{\Theta} = -\nabla_{\Theta}\mathcal{L} = -\Big(\frac{\partial \mathcal{F}^e(\Theta)}{\partial \Theta}\Big)\mathcal{R}(\Theta).
\]
The parameters of each model $j$ in the ensemble evolve according to:
\[\label{dot}
\dot{\theta}_{j} = -\nabla_{\theta_{j}}\mathcal{L} = -\Big(\frac{\partial \mathcal{F}^e(\Theta)}{\partial \theta_{j}}\Big) \mathcal{R}(\Theta) = -\frac{1}{\sqrt{m}}\Big(\frac{\partial \mathcal{F}_n(\theta_{j})}{\partial \theta_{j}}\Big) \mathcal{R}(\Theta).
\]
The time derivative operator $\frac{\partial}{\partial t}$ at $t=0$ can be expanded as follows:
\[\label{time_der}
\frac{\partial}{\partial t} = \Big\langle \dot{\Theta},\frac{\partial }{\partial \Theta}\Big\rangle = \sum_{j=1}^m \Big\langle\dot{\theta}_{j},\frac{\partial}{\partial \theta_{j}}\Big\rangle
\]
Plugging the definition of $\dot{\theta}_{j}$ in Eq.~\ref{dot} into Eq.~\ref{time_der} yields:
\[
\frac{\partial}{\partial t} = -\frac{1}{\sqrt{m}} \mathcal{R}(\Theta) \sum_{j=1}^m\Big\langle \frac{\partial \mathcal{F}_n(\theta_{j})}{\partial \theta_{j}}, \frac{\partial }{\partial \theta_{j}}\Big \rangle = -\frac{1}{\sqrt{m}} \mathcal{R}(\Theta) \sum_{j=1}^m\hat{\Gamma}_{j}
\]
where we have introduced the operator $\hat{\Gamma}_{j}= \langle\frac{\partial \mathcal{F}_n(\theta_{j})}{\partial \theta_{j}}, \frac{\partial }{\partial \theta_{j}}\rangle$.
For each model $j = j_0$ in the ensemble, the $r$-th time derivative of its corresponding NTK at $t=0$ is therefore given by:
\[\label{expand}
\Big(\frac{\partial}{\partial t} \Big)^r\mathcal{K}_n(\theta_{j_0}) &=& \left[-\frac{1}{\sqrt{m}} \mathcal{R}(\Theta) \Big(\sum_{j=1}^m\hat{\Gamma}_{j}\Big) \right]^r\mathcal{K}_n(\theta_{j_0})\\
&=& \left[-\frac{1}{m} \Big(\sum_{j_1,j=1}^m\mathcal{R}_{j_1} \hat{\Gamma}_{j}\Big) \right]^r\mathcal{K}_n(\theta_{j_0})
\]
Denoting $\overset{\curvearrowleft}{\prod}_{u=1}^r A_u = A_rA_{r-1}...A_1=\big[A\big]^r$, we can now expand the $r$-th derivative: 
\[
\Big(\frac{\partial}{\partial t} \Big)^r\mathcal{K}_n(\theta_{i_0})=
\Big(\frac{-1}{m}\Big)^r  \sum_{j_1...j_r=1}^m\left[\overset{\curvearrowleft}{\prod}_{u=1}^{r}  \Big(\mathcal{R}_{j_{u}} \sum_{j=1}^{m}\hat{\Gamma}_{j}\Big)\right]\mathcal{K}_n(\theta_{j_0})
= \Big(\frac{-1}{m}\Big)^r  \sum_{j_1...j_r=1}^m\xi_{j_0...j_r}
\]
where $\xi_{j_0...j_r} = \Big[\overset{\curvearrowleft}{\prod}_{u=1}^{r}  \Big(\mathcal{R}_{j_{u}} \sum_{j=1}^{m}\hat{\Gamma}_{j}\Big)\Big]\mathcal{K}_n(\theta_{j_0})$.\\
Using the notation $\forall_{j,u},~~(\hat{\Gamma}_j)^u \mathcal{R}_j = \mathcal{R}_j^{(u)}$, and noticing that $\mathcal{K}_n(\theta_{j}),\mathcal{R}_j$ depend only on $\theta_{j}$, the following hold:
\begin{align}
    & \forall_{j_1 \neq j_2},~~~ \hat{\Gamma}_{j_1} \mathcal{K}_n(\theta_{j_2}) = 0.\label{one}\\
    & \forall_{u,j_1 \neq j_2},~~~ \hat{\Gamma}_{j_1} \mathcal{R}_{j_2}^{(u)} = 0.\label{two}\\
    & \forall_j,~~~\hat{\Gamma}_{j} \mathcal{R}_j = \mathcal{K}_n(\theta_{j}). \label{three}\\
    & \forall_{j,u,v},~~~\hat{\Gamma}_{j}\mathcal{R}_j^{(u)}\mathcal{R}_j^{(v)} = \mathcal{R}_j^{(u+1)}\mathcal{R}_j^{(v)} + \mathcal{R}_j^{(u)}\mathcal{R}_j^{(v+1)}.\label{four}
\end{align}
where Eq.~\ref{four} is the application of the chain rule.

Using the above equalities, the terms $\xi_{j_0...j_r}$ can be expressed as a sum over a finite set $\bold{S}^r$ as follows:
\[\label{xi_form}
\forall_{j_0,,,j_r},~~~\xi_{j_0...j_r}
= \sum_{\bold{S}^r} \prod_{v=0}^r\mathcal{R}_{j_v}^{(u_v)},~~~~\bold{S}^r:=\{u_v\}_{v=0}^r\Big|  \substack{\forall_{0<v},0\leq u_v \leq r-v\\
2 \leq u_0 \leq r+1\\
\sum_{v=0}^r u_v=r+1}
\]
\paragraph{Example:} for $r=2$, the term $\xi_{j_0,j_1,j_2}$ is expanded as follows:
\[
\xi_{j_0,j_1,j_2} &=& \Big[\overset{\curvearrowleft}{\prod}_{u=1}^{2}  \Big(\mathcal{R}_{j_{u}} \sum_{j=1}^{m}\hat{\Gamma}_{j}\Big)\Big]\mathcal{K}_n(\theta_{j_0})\]
Expanding the multiplication and using Eq.~\ref{three}:
\[
\xi_{j_0,j_1,j_2} &=&  \Big(\mathcal{R}_{j_{2}} \sum_{j=1}^{m}\hat{\Gamma}_{j}\Big)\Big(\mathcal{R}_{j_{1}} \sum_{j=1}^{m}\hat{\Gamma}_{j}\Big)\mathcal{R}_{j_0}^{(1)}
\]
Using the chain rule in Eq.\ref{four}, and eliminating elements using Eq.~\ref{two}:
\[
\xi_{j_0,j_1,j_2} &=& \Big(\mathcal{R}_{j_{2}} \sum_{j=1}^{m}\hat{\Gamma}_{j}\Big)\mathcal{R}_{j_1}\mathcal{R}_{j_0}^{(2)}
= \mathcal{R}_{j_{2}} \mathcal{R}_{j_1}^{(1)}\mathcal{R}_{j_0}^{(2)} + \mathcal{R}_{j_{2}} \mathcal{R}_{j_1}\mathcal{R}_{j_0}^{(3)}
\]
We can now express the result in the formulation of Eq.~\ref{xi_form}
\[
\xi_{j_0,j_1,j_2} &=& \sum_{\bold{S}^2} \prod_{v=0}^2\mathcal{R}_{j_v}^{(u_v)},~~~~\bold{S}^2:=\{u_v\}_{v=0}^2\Big|\substack{\forall_{0<v},0\leq u_v \leq 2-v\\
2 \leq u_0 \leq 3\\
\sum_{v=0}^2 u_v=3}
\]

The $r$'th time derivative of the full ensemble $\mathcal{K}^e$ is given by:
\[
\Big(\frac{\partial}{\partial t} \Big)^r\mathcal{K}^e &=& \frac{1}{m}\sum_{j_0=1}^m \Big(\frac{\partial}{\partial t} \Big)^r\mathcal{K}_n(\theta_{j_0})\\
&=&  \frac{(-1)^r}{m^{r+1}}\sum_{j_0...j_r=1}^m \xi_{j_0...j_r}\\
&=&  \frac{(-1)^r}{m^{r+1}}\sum_{j_0...j_r=1}^m \sum_{\bold{S}^r} ~~\prod_{v=0}^r\mathcal{R}_{j_v}^{(u_v)}\\
&=&\label{full} (-1)^r \sum_{\bold{S}^r} \prod_{v=0}^r \Big(\frac{\sum_{j=1}^m\mathcal{R}_{j}^{(u_v)}}{m}\Big)
\]


Note that for $0\leq u$, the term $\mathcal{R}_{j}^{(u+1)}$ represents the $u$'th time derivative of $\mathcal{K}_n(\theta_j)$ under gradient flow with the loss $\mathcal{L} = \mathcal{R}_j$. Using results\footnote{In \cite{guy}, the $\mathcal{O}_p(n^{-1})$ result was obtained using a conjecture, and demonstrated empirically to be tight. An $\mathcal{O}_p(n^{-0.5})$ result of the same quantity is obtained rigorously in \cite{NTK}, which yields an asymptotic behaviour of $\mathcal{O}_p(n^{-0.5}m^{-1})$ for the ensemble.} from \cite{guy} on wide single fully connected models, we have that $\forall_{u > 1}~~\mathcal{R}_j^{(u)} \sim \mathcal{O}_p(n^{-1})$, and $\mathcal{R}_j^{(1)} \sim \mathcal{O}_p(1)$. Moreover, from the independence of the weights $\forall_{j_1 \neq j_2},~~\theta_{j_1}\indep \theta_{j_2}$, it holds that $\forall_{u,j_1\neq j_2},~~ \mathcal{R}_{j_1}^u \indep \mathcal{R}_{j_2}^u$. Therefore, for any fixed $r$ we may apply the central limit theorem for large $m$ on the terms in Eq.~\ref{full} individually:
\[
 \Big(\frac{\sum_{j=1}^m\mathcal{R}_{j}^{(u)}}{m}\Big) \sim \begin{cases}
 \mathcal{O}_p(\frac{1}{\sqrt{m}n}) & u> 1\\
 \mathcal{O}_p(1) & u = 1 \\
 \mathcal{O}_p(\frac{1}{\sqrt{m}}) & u = 0
 \end{cases}
\]
Plugging back into Eq.~\ref{full} yields the desired result by noticing that $2 \leq u_0$ and $u_r = 0$:
\[
\Big(\frac{\partial}{\partial t} \Big)^r\mathcal{K}^e \sim \mathcal{O}_p(\frac{1}{nm})
\]
\end{proof}

\end{document}